\documentclass[twoside]{article}

\usepackage[accepted]{aistats2022}
\usepackage[round]{natbib}
\usepackage{algorithm}
\usepackage[noend]{algpseudocode}
\usepackage[T1]{fontenc}
\usepackage[utf8]{inputenc}
\usepackage[english]{babel}
\usepackage{amsthm}
\usepackage{amsmath}
\usepackage{amssymb}
\usepackage{graphicx}
\usepackage{dsfont}
\usepackage{amsmath}
\usepackage{array}
\usepackage{multirow}
\usepackage{caption}
\usepackage{color}
\usepackage{multicol}

\theoremstyle{plain}

\theoremstyle{plain}
\newtheorem{prop}{\protect\propositionname}
\theoremstyle{plain}
\newtheorem{lem}{\protect\lemmaname}
\theoremstyle{plain}
\newtheorem{thm}{\protect\theoremname}
\theoremstyle{plain}
\newtheorem{cor}{\protect\corollaryname}  
\theoremstyle{definition}

\theoremstyle{definition}

\theoremstyle{definition}
\newtheorem{rem}{\protect\remarkname}

\makeatother

\providecommand{\claimname}{Claim}
\providecommand{\lemmaname}{Lemma}
\providecommand{\propositionname}{Proposition}
\providecommand{\theoremname}{Theorem}
\providecommand{\corollaryname}{Corollary} 
\providecommand{\definitionname}{Definition}
\providecommand{\assumptionname}{Assumption}
\providecommand{\remarkname}{Remark}

\DeclareMathOperator*{\argmax}{arg\,max}

\newcommand{\comment}[1]{}

\newcommand{\Xv}{\mathbf{X}}
\newcommand{\yv}{\mathbf{y}}

\newcommand{\Fc}{\mathcal{F}}

\newcommand{\Xc}{\mathcal{X}}

\begin{document}

%

%

\twocolumn[

\aistatstitle{Gaussian Process Bandit Optimization with Few Batches}

\aistatsauthor{ Zihan Li \And Jonathan Scarlett }

\aistatsaddress{ National University of Singapore \And  National University of Singapore } ]

\begin{abstract}
In this paper, we consider the problem of black-box optimization using Gaussian Process (GP) bandit optimization with a small number of batches. Assuming the unknown function has a low norm in the Reproducing Kernel Hilbert Space (RKHS), we introduce a batch algorithm inspired by batched finite-arm bandit algorithms, and show that it achieves the cumulative regret upper bound $O^\ast(\sqrt{T\gamma_T})$ using $O(\log\log T)$ batches within time horizon $T$, where the $O^\ast(\cdot)$ notation hides dimension-independent logarithmic factors and $\gamma_T$ is the maximum information gain associated with the kernel.  This bound is near-optimal for several kernels of interest and improves on the typical $O^\ast(\sqrt{T}\gamma_T)$ bound, and our approach is arguably the simplest among algorithms attaining this improvement.  In addition, in the case of a constant number of batches (not depending on $T$), we propose a modified version of our algorithm, and characterize how the regret is impacted by the number of batches, focusing on the squared exponential and Mat\'ern kernels.  The algorithmic upper bounds are shown to be nearly minimax optimal via analogous algorithm-independent lower bounds.
\end{abstract}

\section{INTRODUCTION}

Black-box optimization is the problem of maximizing an unknown function using only point queries that are typically expensive to evaluate.  This problem has a wide range of applications including hyperparameter tuning \citep{Sno12}, experimental design \citep{Sri09}, recommendation system \citep{Van14}, and robotics \citep{Liz07}. For example, tuning the hyperparameters of a machine learning model can be considered as a black-box optimization problem, where an evaluation of the function $f$ at input $x$ trains the model on hyperparameters $x$, and returns the cross-validation accuracy $f(x)$ on the validation set. Moreover, large models would make the evaluations expensive.  Gaussian Process (GP) optimization is a prominent technique for maximizing black-box functions. The main idea is to place a prior over the unknown function and update the posterior distribution based on the noisy observations. 

In this paper, we consider the problem of GP optimization in the case that observations are performed in batches, rather than being fully sequential.  This is known to provide considerable benefits in settings where it is possible to perform multiple queries in parallel.  The key distinction in our work compared to earlier works is that we consider the number of batches to be very low, namely, either independent of time horizon $T$, or growing very slowly as $O(\log \log T)$.

\subsection{Related Work}

Theoretical studies of GP optimization problems can be categorized into the Bayesian setting (i.e., the function is assumed to be drawn from a GP) or the non-Bayesian Reproducing Kernel Hilbert Space (RKHS) setting (i.e., the function has bounded RKHS norm).  In this paper, we focus on the latter, which is often known as GP bandit optimization.

Classical GP optimization methods such as Gaussian process upper confidence bound (GP-UCB) \citep{Sri09}, expected improvement (EI) \citep{Mockus78}, and Thompson sampling (TS) \citep{Tho33} iteratively choose points by optimizing their acquisition functions, and typically have a cumulative regret guarantee of $O^\ast(\sqrt{T}\gamma_T)$ under the RKHS setting, where the $O^\ast(\cdot)$ notation hides dimension-independent logarithmic factors, and $\gamma_T$ is the maximum information gain associated with the kernel. 

The $O^\ast(\sqrt{T}\gamma_T)$ scaling is not tight in general, and can even fail to be sublinear, including for the Mat\'ern kernel with broad choices of the smoothness $\nu$ and dimension $d$.  The SupKernelUCB algorithm \citep{Val13} overcomes this limitation and attains an $O^\ast(\sqrt{T\gamma_T})$ guarantee on discrete domains, which extends to continuous domain under mild Lipschitz-style assumptions.  In view of recent improved bounds on $\gamma_T$ \citep{Vak20a}, this nearly matches algorithm-independent lower bounds for the squared exponential (SE) and Mat\'ern kernels \citep{Sca17a}.

SupKernelUCB is not considered to be practical \citep{Cal19}, and a more practical algorithm guaranteeing sublinear regret for the Mat\'ern kernel was given by \cite{Jan20}, albeit with suboptimal scaling.  More recently, two more approaches were shown to attain $O^\ast(\sqrt{T\gamma_T})$ scaling: Robust Inverse Propensity Score (RIPS) for experimental design \citep{Jamieson21}, and a tree-based domain-shrinking algorithm known as GP-ThreDS \citep{Salgia2020}.  GP-ThreDS is shown to perform well empirically, while an advantage of the RIPS approach is that it requires at most $O(\log T)$ batches and is robust to model misspecification.


Several works have also proposed batch algorithms for GP optimization, including GP-UCB-PE \citep{Con13}, GP-BUCB \citep{Des14a}, and SynTS \citep{Kan18}.  The focus in these works is on {\em fixed batch sizes}, and when the batch size is held fixed, the number of batches scales as $\Theta(T)$, in stark contrast to our focus.  On the other hand, if a constant number of batches is used and all batches have equal size, then the first batch alone contributes $\Omega(T)$ regret.  To overcome this, we will crucially rely on {\em variable-length} batches.  While our algorithm and its analysis are new, some algorithms exist that also exploit varying batch lengths; for instance, the above-mentioned work of \cite{Jamieson21} also falls in this category, as does the practical B3O algorithm that chooses points according to the peaks of an infinite mixture model \citep{Nguyen16}.

While GP-BUCB enjoys a cumulative regret of $O^\ast(\sqrt{T}\gamma_T)$ regardless of the batch size, this is only the case after having access to a suitably large offline initialization stage sampling $T^\text{init}$ points that are not counted in the regret.  If these points instead do contribute to the regret, then the bound becomes  $O^\ast(T^\text{init}+\sqrt{T}\gamma_T)$, which may be prohibitively large.

Our work builds on studies of batched algorithms for multi-armed bandit (MAB) problems \citep{Cesa13,Esf20}.  In particular, \cite{Cesa13} attains a near-optimal bound using $O(\log \log T)$ batches, and \cite{Esf20} considers a constant number of batches.  Our goal is to provide analogous results for the RKHS setting, and while we build on ideas from these MAB works, the details are substantially different.

\subsection{Contributions}

We introduce a batch algorithm Batched Pure Exploration (BPE) for GP bandits, and show that it achieves the cumulative regret upper bound $O^\ast(\sqrt{T\gamma_T})$ within time horizon $T$ using $O(\log\log T)$ batches.  As outlined above, this is near-optimal for the SE and Mat\'ern kernels.  Compared to the alternatives SupKernelUCB, RIPS, and GP-ThreDS outlined above, we believe that our approach is the simplest (see Remark \ref{rem:comparison} below for discussion), and only relies on a basic GP-based confidence bound \citep{Vakili21} without the need for the less elementary tools used previously.

In addition, in the case of a constant number of batches (i.e., not increasing with $T$), we propose a modified version of BPE, and 
characterize how the regret is impacted by the number of batches for the SE and Mat\'ern kernels.  Our bound is shown to be near-optimal by providing an analogous algorithm-independent lower bound.

\section{PROBLEM SETUP}
\label{sec:setup}

We consider the problem of optimizing a black-box function $f(x)$ over a compact domain $\mathcal{X}\subseteq\mathbb{R}^d$ with batched observations (possibly with varying batch lengths). The algorithm selects a single point $x_t$ at time $t=1,\dots,T$, where the time horizon $T$ is split into $B$ batches indexed by $i=1,\dots,B$. The length of the $i$-th batch is denoted by $N_i$, and we use $t_i=\sum_{j=1}^i N_j$ to denote the time at the end of batch $i$. At the end of each batch, the algorithm observes the batch of noisy samples $y_t=f(x_t)+\epsilon_t$, where the noise terms $\epsilon_t$ are i.i.d. $R$-sub-Gaussian for all $t$.

We assume that $f\in\mathcal{F}_k(\Psi)$, where $\mathcal{F}_k(\Psi)$ denotes the set of all functions whose RKHS norm $\| f \|_k$ is upper bounded by some constant $\Psi>0$, and we assume the kernel is normalized such that $k(x,x)\leq 1,\forall x\in \mathcal{X}$.  We consider an algorithm that uses a fictitious\footnote{By ``fictitious'', we mean that the Bayesian model is only used by the algorithm, but the theoretical analysis itself is for the non-Bayesian RKHS model.} Bayesian GP model in which $f$ is considered to be randomly drawn from a GP with prior mean $\mu_0(x)=0$ and kernel function $k$. Given a sequence of points $(x_1,\dots,x_t)$ and their noisy observations $\mathbf{y}_t=(y_1,\dots,y_t)$ up to time $t$, the posterior distribution of $f$ is a GP with 
mean and variance given by
\begin{align}
	\mu_t(x) &= \mathbf{k}_t(x)^T(\mathbf{K}_t + \lambda\mathbf{I}_t)^{-1}\mathbf{y}_t \label{eq:mean}\\
	\sigma_t(x)^2  &= k(x,x) - \mathbf{k}_t(x)^T(\mathbf{K}_t + \lambda\mathbf{I}_t)^{-1}\mathbf{k}_t(x),\label{eq:var}
\end{align}
where $\mathbf{k}_t(x) = [k(x_i, x)]_{i=1}^{t}$, $\mathbf{K}_t = [k(x_t, x_{t'})]_{t,t'}$, and $\lambda>0$ is a free parameter. Similarly, we denote the posterior mean and variance given only the points and observations in batch $i$ by $\mu^i(x)$ and $\sigma^i(x)^2$.

We pay particular attention to two commonly-used kernels, namely, the squared exponential (SE) kernel and Mat\'ern kernel:
\begin{align*}
	k_\text{SE}(x,x')&=\exp\Bigg(-\frac{(d_{x,x'})^2}{2l^2}\Bigg), \\
	k_\text{Mat\'ern}(x,x')&=\frac{2^{1-\nu}}{\Gamma(\nu)}\Bigg(\frac{\sqrt{2\nu}d_{x,x'}}{l}\Bigg)^\nu B_\nu \Bigg(\frac{\sqrt{2\nu}d_{x,x'}}{l}\Bigg),
\end{align*}
where $d_{x,x'}=\| x-x'\|$, $l>0$ denotes the length-scale, $\nu$ is a smoothness parameter, and $\Gamma$ and $B_\nu$ are the Gamma and Bessel functions.  We denote the \textit{maximum information gain} by \citep{Sri09}
\begin{align*}
	\gamma_t&=\max_{A\subseteq \mathcal{X}:|A|=t}I(f_A;y_A)\\
	&=\max_{x_1,\dots,x_t}\frac{1}{2}\log\det(\mathbf{I}_t+\lambda^{-1}\mathbf{K}_t),
\end{align*}
where $f_A=[f(x_t)]_{x_t\in A}$, $y_A=[y_t]_{x_t\in A}$, and $I(\cdot;\cdot)$ denotes mutual information \citep{Cov01}. $\gamma_t$ quantifies the maximum reduction in uncertainty about $f$ after $t$ observations. We will use the following known upper bounds on $\gamma_t$ from \citep{Sri09,Vak20a} in our analysis:
\begin{align}
	\gamma_t^\text{SE} &= O^\ast((\log t)^d), \label{eq:se}\\
	\gamma_t^\text{Mat\'ern} &= O^\ast\Big(t^\frac{d}{2\nu+d}\Big), \label{eq:mat}
\end{align}
where the $O^\ast(\cdot)$ notation suppresses dimension-independent logarithmic factors.

To circumvent the $\sqrt{T} \gamma_T$ bottleneck discussed above, we avoid using standard confidence bounds \citep{Sri09,Cho17}, and instead use a recent bound by \cite{Vakili21}.\footnote{Similar improved confidence bounds for non-adaptive sampling are also implicit in \citep{Val13} and \citep{Jamieson21}, but we find those of \citep{Vakili21} more directly suitable for our purposes. }  In view of these bounds only being valid for offline (i.e., non-adaptive) sampling strategies, we use an upper confidence bound (UCB) and lower confidence bound (LCB) defined using {\em only points within a single batch}:
\begin{align}
	\mathrm{UCB}_i(x) &= \mu^{i}(x) + \sqrt{\beta}\sigma^{i}(x), \label{eq:ucb}\\
	\mathrm{LCB}_i(x) &= \mu^{i}(x) - \sqrt{\beta}\sigma^{i}(x), \label{eq:lcb}
\end{align}
where $i \in \{1,\dotsc,B\}$ indexes the batch number, and $\beta$ is an exploration parameter.  The choice of $\beta$ is dictated by the following result, which we emphasize only holds for points chosen in an offline manner (i.e., independent of the observations).

\begin{lem}\textup{\citep[Theorem 1]{Vakili21}} \label{lem:conf}
    Let $f$ be a function with RKHS norm at most $\Psi$, and let ${\mu}_t(x)$ and ${\sigma}_t(x)^2$ be the posterior mean and variance based on $t$ points $\Xv = (x_1,\dotsc,x_t)$ with observations $\yv = (y_1,\dotsc,y_t)$.  Moreover, suppose that the $t$ points in $\Xv$ are chosen independently of all samples in $\yv$.  Then, for any fixed $x\in \mathcal{X}$, it holds with probability at least $1-\delta$ that $|f(x)-\mu_t(x)|\leq\sqrt{\beta(\delta)}\sigma_t(x)$, where $\beta(\delta) = \big(\Psi + \frac{R}{\sqrt{\lambda}} \sqrt{2 \log \frac{1}{\delta}}\big)^2$.
\end{lem}

In the case of a finite domain, a union bound over all possible $x$ immediately gives the following.

\begin{cor}\label{cor:beta}
    Under the setup of Lemma \ref{lem:conf}, if the domain $\Xc$ is finite, then for any $\delta\in(0,1)$, it holds with probability at least $1-\delta$ that $|f(x)-\mu_t(x)|\leq\sqrt{\beta(\frac{\delta}{|\Xc|})}\sigma_t(x)$ simultaneously for all $x\in \mathcal{X}$.
\end{cor}

To evaluate the algorithm performance, with $x^\ast=\argmax_{x\in \mathcal{X}}f(x)$ being the optimal point, we let $r_t=f(x^\ast)-f(x_t)$ denote the simple regret incurred at time $t$, $R^i=\sum_{t=t_{i-1}+1}^{t_i}r_t$ denote the cumulative regret incurred in batch $i$, and $R_T=\sum_{i=1}^B R^i$ denote the overall cumulative regret incurred. Our objective is to attain small $R_T$ with high probability. 

Throughout the paper, we assume that the kernel parameters $(\nu,l)$ and dimension $d$ are fixed constants (i.e., not depending on $T$).

\section{SLOWLY GROWING NUMBER OF BATCHES}

\begin{algorithm}[!t]
\caption{Batched Pure Exploration (BPE) for the case that the domain $\Xc$ is a finite set \label{algo:bpe}}
\begin{algorithmic}[1]
    \Require Domain $\mathcal{X}$, GP prior $(\mu_0, \sigma_0)$, time horizon $T$
    \State $\mathcal{X}_1 \gets \mathcal{X},t \gets 1,N_0\gets 1$
    \For{$i \gets 1, 2, \dots$} 
        \State $N_i\gets\big\lceil\sqrt{TN_{i-1}} \big\rceil$
        \State $\mathcal{S}_i\gets\emptyset$
        \For{$j \gets 1, 2, \dots, N_i$}
            \State Compute $\sigma^i_{j-1}$ only based on $\mathcal{S}_i$
            \State $x_t \gets \argmax_{x\in \mathcal{X}_{i}} \sigma^i_{j-1}(x)$
            \State $\mathcal{S}_i\gets \mathcal{S}_i\cup\{x_t\}$
            \State $t \gets t+1$
            \If{$t>T$}
            \State{Terminate}
        \EndIf
        \EndFor
        \State Observe all points in $\mathcal{S}_i$
        \State Compute $\mu^{i}$ and $\sigma^{i}$ only based on $\mathcal{S}_i$
        \State $\mathcal{X}_{i+1} \gets \{ x \in \mathcal{X}_{i}: \mathrm{UCB}_i(x) \geq \max \mathrm{LCB}_i \}$, where the confidence bounds use \eqref{eq:ucb} and \eqref{eq:lcb} with $\beta = \big( \Psi + \frac{R}{\sqrt{\lambda}} \sqrt{2 \log \frac{|\Xc| B}{\delta}} \big)^2$ (here $B$ is deterministic given $T$, and can be computed using line 3)
    \EndFor   
\end{algorithmic}
\end{algorithm}

In this section, we study the scenario where we are allowed to choose the number of batches $B$ as a function of $T$, and provide an algorithm and its theoretical regret upper bound.  We select the batch sizes in a similar manner to the multi-armed bandit problem \citep{Cesa13} and suitably adapt the analysis.  The resulting algorithm, BPE, is described in Algorithm \ref{algo:bpe}; we initially focus on the case that $\Xc$ is finite, and later turn to continuous domains.

BPE maintains a set $\mathcal{X}_i$ of potential maximizers, and each batch consists of an exploration step and an elimination step to reduce the size of $\mathcal{X}_i$. In the exploration step, the algorithm explores the most uncertain regions in $\mathcal{X}_i$ by iteratively picking the point with the highest uncertainty, which is measured by the posterior variance. This is possible because the posterior variance can be computed using \eqref{eq:var} without knowing any observed values. To ensure the validity of the offline assumption in Corollary \ref{cor:beta}, we ignore the previous batches and compute the posterior variance only based on the points selected in the current batch.\footnote{The fact that the sub-domain $\mathcal{X}_i$ depends on the samples from previous batches is inconsequential, because we do not use those samples in the new posterior calculation.} The number of points selected in the $i$-th batch is $N_i=\big\lceil\sqrt{TN_{i-1}} \big\rceil$, where we define $N_0=1$ for convenience.  Note that these values of $N_i$ can be computed offline (i.e., in advance), and thus, the total number of batches $B$ is also deterministic given $T$.

In the elimination step, the algorithm eliminates points whose UCB is lower than the highest LCB from $\mathcal{X}_i$ (as is commonly done in MAB and GP bandit problems), as these points must be suboptimal (as long as the confidence bounds are valid). Hence, we only need to explore the updated $\mathcal{X}_i$ in next batch. 

To re-iterate what we discussed previously, compared to regular GP optimization algorithms, our use of the offline confidence bound in Lemma \ref{lem:conf} is the key to attaining a final dependence of $\sqrt{T\gamma_T}$ instead of $\sqrt{T} \gamma_T$.

The following proposition shows that the choice of batch size $N_i$ in the BPE algorithm yields a number of batches given by $B=O(\log\log T)$.
\begin{prop} \label{prop:batch_count}
    With $N_i=\big\lceil\sqrt{TN_{i-1}} \big\rceil$ and $N_0=1$, the BPE algorithm terminates (i.e., reaches $T$ selections) within at most $\lceil\log_2\log_2T\rceil + 1$ batches.
\end{prop}
\begin{proof}
    Define $\widetilde{N}_i=\sqrt{T\widetilde{N}_{i-1}}$ with $\widetilde{N}_0=1$, which implies $\widetilde{N}_i=T^\frac{1}{2}(\widetilde{N}_{i-1})^\frac{1}{2}=T^{\frac{1}{2}+\frac{1}{4}}(\widetilde{N}_{i-2})^\frac{1}{4}=\dots=T^{\frac{1}{2}+\frac{1}{4}+\dots+\frac{1}{2^i}}(\widetilde{N}_0)^\frac{1}{2^i}=T^{1-2^{-i}}$.
    Let $\widetilde{B}=\lceil\log_2\log_2T\rceil$, which implies $\widetilde{N}_{\widetilde{B}} = T/T^{2^{-\widetilde{B}}} \geq T/T^\frac{1}{\log_2T}=T/2$. Since $\sum_{i=1}^{\widetilde{B}+1} N_i\geq N_{\widetilde{B}} + N_{\widetilde{B}+1} \geq\widetilde{N}_{\widetilde{B}} + \widetilde{N}_{\widetilde{B}+1}\geq T$, we must have $B\leq\widetilde{B}+1$.
\end{proof}

Our first main result is stated as follows.

\begin{thm}\label{thm:bpe}
Under the setup of Section \ref{sec:setup} with a finite domain $\Xc$, choosing $\lambda=R^2$, the BPE algorithm yields with probability at least $1-\delta$ that
\begin{align*}
	R_T=O^\ast(\Lambda\sqrt{T\gamma_T}),
\end{align*}
where $\Lambda=\Psi+\sqrt{\log\frac{|\Xc|}{\delta}}$. In particular,
\begin{itemize}
	\item for the SE kernel, $R_T=O^\ast(\Lambda\sqrt{T(\log T)^d})$;
	\item for the Mat\'ern kernel, $R_T=O^\ast(\Lambda T^\frac{\nu+d}{2\nu+d})$.
\end{itemize}
\end{thm}

The proof is given in Appendix \ref{sec:proof_bpe}, and is based on the fact that after performing maximum-uncertainty sampling for $N$ rounds, we are guaranteed to uniformly shrink the uncertainty to $O\big( \sqrt{\frac{\gamma_N}{N}} \big)$.  From this result, our choices of batch lengths turn out to provide a per-batch cumulative regret of $O(\sqrt{T\gamma_T})$, which gives the desired result.

\begin{rem} \label{rem:comparison}
    As discussed previously, we have provided a new approach to attaining the improved $O^\ast(\sqrt{T\gamma_T})$ dependence, and we are the first to do so using $O(\log \log T)$ batches, improving on the $O(\log T)$ batches used by \cite{Jamieson21}.  Perhaps more importantly, we believe that the simplicity of our algorithm and analysis are particularly desirable compared to the existing works attaining $O^\ast(\sqrt{T\gamma_T})$ scaling \citep{Val13,Salgia2020,Jamieson21}.  Briefly, some less elementary techniques used in these works include adaptively partitioning the time indices into carefully-chosen subsets of $\{1,\dotsc,T\}$ \citep{Val13}, tree-based adaptive partitioning of $\Xc$ \citep{Salgia2020}, and experimental design over the $|\Xc|$-dimensional simplex by \cite{Jamieson21}.
\end{rem}

\subsection{Extension to Continuous Domains} \label{sec:continuous}

While we have focused on discrete domains for simplicity, extending to continuous domains is straightforward.  For concreteness, here we discuss the case that $\Xc=[0,1]^d$.

As discussed by \cite{Jan20}, the simplest approach to establishing analogous regret guarantees in this case is to construct a finite subdomain $\widetilde{\Xc}$ of $T^{d/2}$ points, with equal spacing of width $\frac{1}{\sqrt T}$ in each dimension.  If we apply Algorithm \ref{algo:bpe} on the finite domain $\widetilde{\Xc}$, then the regret guarantee in Theorem \ref{thm:bpe} immediately holds with respect to the restricted-domain function, with $\Lambda=\Psi+\sqrt{\frac{d}{2}\log T + \log\frac{1}{\delta}}$.  To convert such a guarantee to the entire continuous function, we can use the following lemma.

\begin{lem}\textup{\citep[Proposition 1 and Remark 5]{Shekhar20}}\label{lem:lipschitz}
For $f\in\mathcal{F}_k(\Psi)$ on a compact continuous domain, if $k$ is the SE or Mat\'ern kernel (with $\nu > 1$) and $\Psi$ is constant, then $f$ is guaranteed to be Lipschitz continuous with some constant $L$ depending only on the kernel parameters.
\end{lem}

This Lipschitz continuity property establishes that $f(x^\ast)=\max_{x\in\Xc}f(x)\leq\max_{\widetilde{x}\in\widetilde{\Xc}}f(\widetilde{x})+ O\big( \frac{L}{\sqrt T} \big)$, which implies that to convert the discrete-domain cumulative regret to the continuous domain, we only need to add $O(L\sqrt{T})$ to the upper bound in Theorem \ref{thm:bpe}, which has no impact on the scaling laws for fixed $L$.

A more sophisticated approach to handling continuous domains is given by \cite{Vakili21}, based on forming a fine discretization {\em purely for the theoretical analysis} and not for use by the algorithm itself.  Such an approach can also be adopted here, and leads to a regret bound with the same scaling as the direct approach outlined above.\footnote{This is under the assumption that the algorithm can exactly optimize its acquisition function over the continuous domain, and over any sub-domain formed as a result of action elimination.  While this is typically not to be expected in practice, our analysis also goes through essentially unchanged when the algorithm is only required to find a point whose posterior variance is within a constant factor (e.g., $\frac{1}{2}$) of the maximum.}   The details of this approach are somewhat more technical than the basic approach above.  In particular, as observed in \citep{Vakili21}, we need to not only establish that the discretization has a minimal impact on $f$, but also on the posterior mean.  A slight difference compared to \citep{Vakili21} is that we additionally need to consider the impact of the discretization on the posterior standard deviation, since it is used in the elimination step of our algorithm.  The full details be found in a recent paper that builds on our work while incorporating sparse GP approximations \citep{Vakili22}; the analysis therein equally applies to the case that an exact GP model is used.

\section{CONSTANT NUMBER OF BATCHES}

While $B=O(\log\log T)$ is rather small, it is also of significant interest to understand scenarios where the number of batches is even further constrained.  In this section, we consider the case that $B$ is a pre-specified constant (not depending on $T$).

Throughout the section, we focus our attention on algorithms for which the batch lengths are pre-specified, i.e., the algorithm {\em cannot} vary the batch lengths based on previous observations.  For the algorithmic upper bounds such a property is clearly a strength, whereas for the algorithm-independent lower bounds, this leaves open the possibility of adaptively-chosen batch sizes further improving the performance.  We discuss the challenges regarding lower bounds for adaptively-chosen batch lengths in Appendix \ref{sec:variable}.

\subsection{Upper Bound}

We adapt BPE by carefully modifying the batch sizes, and in contrast to Theorem \ref{thm:bpe}, we do this in a kernel-dependent manner.  The resulting bounds for the SE and Mat\'ern kernels are stated as follows.

\begin{thm}\label{thm:upper}
Under the setup of Section \ref{sec:setup} with an arbitrary constant number of batches $B\geq 2$,
\begin{itemize}
    \item for the SE kernel, let $N_i =\big\lceil (\frac{T}{(\log T)^d})^\frac{1-\eta^i}{1-\eta^B}(\log T)^d\big\rceil$, where $\eta=\frac{1}{2}$;
    \item for the Mat\'ern kernel, let $N_i =\big\lceil T^\frac{1-\eta^i}{1-\eta^B}\big\rceil$, where $\eta = \frac{\nu}{2\nu+d}$;
\end{itemize}
for $i=1,\dots,B-1$, and $N_B=T-\sum_{j=1}^{B-1}N_j$.
Then, choosing $\lambda=R^2$, the corresponding modified BPE algorithm yields with probability at least $1-\delta$ that
\begin{itemize}
    \item for the SE kernel, it holds that $R_T=O^\ast\Big(\Lambda T^\frac{1-\eta}{1-\eta^B}(\log T)^{\frac{\eta-\eta^B}{1-\eta^B}d}\Big)=O^\ast\Big(\Lambda T^\frac{2^B-2^{B-1}}{2^B-1}(\log T)^{\frac{2^{B-1}-1}{2^B-1}d}\Big)$;
    \item for the Mat\'ern kernel, $R_T=O^\ast\Big(\Lambda T^\frac{1-\eta}{1-\eta^B}\Big)$;
\end{itemize}
where $\Lambda=\Psi+\sqrt{\log(|\mathcal{X}|B/\delta)}$.
\end{thm}

The proof is given in Appendix \ref{sec:pf_fixed_B}, and is again based on the fact that $N$ rounds of uncertainty sampling reduces the maximum uncertainty to $O\big( \sqrt{\frac{\gamma_N}{N}} \big)$.  In this case, unlike Theorem \ref{thm:bpe}, the batch lengths $\{N_i\}_{i=1}^B$ are specifically chosen to ensure that roughly the same (upper bound on) regret is incurred in every batch.

We observe that in both cases, the leading term in the upper bound is $T^{\frac{1-\eta}{1-\eta^B}}$, which decreases as $B$ increases.  The limiting leading term of $T^{1-\eta}$ as $B \to \infty$ is also consistent with Theorem \ref{thm:bpe}.

Using the argument given in Section \ref{sec:continuous}, the regret bound in Theorem \ref{thm:upper} again extends directly to continuous domains.

\subsection{Lower Bound}

Since our focus is on the dependence on $T$ and $B$, we assume in the following that the RKHS norm bound $\Psi$ is a fixed constant.  Recall also that the kernel parameters $(\nu,l)$ and dimension $d$ are considered to be fixed constants.  Throughout this subsection, we consider the standard rectangular domain $\Xc = [0,1]^d$ in order to utilize existing lower bounds that were proved for this domain.

While our focus is on the cumulative regret, an existing lower bound on the simple regret (for general algorithms without constraints on the number of batches) turns out to serve as a useful stepping stone.  The simple regret is defined as $r^T=f(x^\ast)-f(x^T)$, where $x^T$ is an additional point returned after $T$ rounds.  Then, we have the following.

\begin{lem}\textup{\citep[Theorem 1]{Cai20}}\label{lem:lower}
Fix $\epsilon\in (0, \frac{1}{2})$,  $\Psi>0$, and $T\in\mathbb{Z}$. Suppose there exists an algorithm that, for any $f\in\mathcal{F}_k(\Psi)$, achieves average simple regret $\mathbb{E}[r^T]\leq \epsilon$ with probability at least $\frac{3}{4}$.\footnote{The value $\frac{3}{4}$ can be replaced by any fixed constant in $(\frac{2}{3},1)$, and doing so does not affect the scaling laws.} Then, if $\frac{\epsilon}{\Psi}$ is sufficiently small, we have
\begin{enumerate}
    \item for the SE kernel, it is necessary that
    \begin{align*}
        T = \Omega\bigg(\frac{1}{\epsilon^2}\Big(\log\frac{1}{\epsilon}\Big)^{d/2}\bigg),
    \end{align*}
    \item for the Mat\'ern kernel, it is necessary that
    \begin{align*}
        T= \Omega\bigg(\frac{1}{\epsilon^{2+d/\nu}}\bigg).
    \end{align*}
\end{enumerate}
\end{lem}

The following corollary provides a lower bound on the cumulative regret in the $m$-th batch, $R^m$, given an upper bound on $t_{m-1}$ and a lower bound on $t_m$, where we recall that $t_i$ is final time index of the $i$-th batch.

\begin{cor}
\label{cor:batch_lower}
Fix constants $a$ and $b$ such that $0<a<b\leq1$, as well as a batch index $m$.  Consider any batch algorithm whose batch sizes are fixed in advance, and satisfy
\begin{enumerate}
    \item for the SE kernel, $t_{m-1}<\big(\frac{T}{(\log T)^{d/2}}\big)^a(\log T)^{d/2}$ and $t_m \geq \big(\frac{T}{(\log T)^{d/2}}\big)^b(\log T)^{d/2}$;
    \item for the Mat\'ern kernel, $t_{m-1}<T^a$ and $t_m \geq T^b$.
\end{enumerate}
Then, there exists $f\in\mathcal{F}_k(\Psi)$ such that with probability at least $\frac{3}{4}$, the cumulative regret incurred in batch $m$ is upper bounded as follows:
\begin{enumerate}
    \item for the SE kernel, $R^m=\Omega\Big((\frac{T}{(\log T)^{d/2}})^{b-a/2}(\log T)^{d/2}\Big)$;
    \item for the Mat\'ern kernel; $R^m=\Omega\Big(T^{b-\frac{\nu}{2\nu+d}a}\Big)$.
\end{enumerate}
\end{cor}
\begin{proof}
For the SE kernel, the inequalities $t_{m-1}<\big(\frac{T}{(\log T)^{d/2}}\big)^a(\log T)^{d/2}$ and $t_m \geq \big(\frac{T}{(\log T)^{d/2}}\big)^b(\log T)^{d/2}$ (with $b>a$) indicate that the size of the $m$-th batch satisfies $N_m \geq C_m \big(\frac{T}{(\log T)^{d/2}}\big)^b(\log T)^{d/2}$ for some constant $C_m$. Now suppose on the contrary that there exists an algorithm with $t_{m-1}$ and $t_m$ bounded as above such that, for any $f\in\mathcal{F}_k(\Psi)$, we have $R^m=o\big(\big(\frac{T}{(\log T)^{d/2}}\big)^{b-a/2}(\log T)^{d/2}\big)$ with probability at least $\frac{3}{4}$. Then, the average regret of points in the $m$-th batch is $\frac{R_m}{N_m}=o\big(\big(\frac{T}{(\log T)^{d/2}}\big)^{-a/2}\big)$ due to the lower bound on $N_m$. Substituting $\epsilon=o\big(\big(\frac{T}{(\log T)^{d/2}}\big)^{-a/2}\big)$ into Lemma \ref{lem:lower},\footnote{When applying Lemma \ref{lem:lower}, we substitute the (single) batch length for $T$, and let $x^T$ be chosen uniformly at random from the set of sampled points in the batch.} we find that this algorithm must have $t_{m-1}=\sum_{i=1}^{m-1} N_i = \omega\big(\big(\frac{T}{(\log T)^{d/2}}\big)^a(\log T)^{d/2}\big)$. However, this contradicts our assumption of $t_{m-1}<\big(\frac{T}{(\log T)^{d/2}}\big)^a(\log T)^{d/2}$. Hence, the claimed lower bound on $R_T$ must hold.

Similarly, for the Mat\'ern kernel, the inequalities $t_{m-1}<T^a$ and $t_m\geq T^b$ indicate the size of the $m$-th batch satisfies $N_m \geq C_m T^b$ for some constant $C_m$. Now suppose on the contrary that there exists an algorithm with $t_{m-1}$ and $t_m$ bounded as above such that, for any $f\in\mathcal{F}_k(\Psi)$, we have  $R^m=o\big(T^{b-\frac{\nu}{2\nu+d}a}\big)$ with probability at least $\frac{3}{4}$. Then, the average simple regret of the $m$-th batch is $\frac{R_m}{N_m}=o\big(T^{-\frac{\nu}{2\nu+d}a}\big)$ due to the lower bound on $N_m$. Substituting $\epsilon=o\big(T^{-\frac{\nu}{2\nu+d}a}\big)$ into Lemma \ref{lem:lower}, we find that this algorithm must have $t_{m-1}=\sum_{i=1}^{m-1} N_i = \omega(T^a)$. However, this contradicts our assumption of $t_{m-1}<T^a$. Hence, the claimed lower bound on $R_T$ must hold.
\end{proof}

Using Corollary \ref{cor:batch_lower} as a building block, we can obtain our main lower bound, stated as follows.

\begin{thm}
\label{thm:lower}
For any batch algorithm with a constant number $B$ of batches and batch sizes fixed in advance, there exists $f\in\mathcal{F}_k(\Psi)$ such that we have with probability at least $\frac{3}{4}$ that
\begin{enumerate}
    \item for the SE kernel, $R_T=\Omega\Big(T^\frac{1-\eta}{1-\eta^B}(\log T)^{\frac{d(\eta-\eta^B)}{2(1-\eta^B)}}\Big)$, where $\eta=\frac{1}{2}$;
    \item for the Mat\'ern kernel, $R_T=\Omega\Big(T^\frac{1-\eta}{1-\eta^B}\Big)$, where $\eta=\frac{\nu}{2\nu+d}$.
\end{enumerate}
\end{thm}

The proof is given in Appendix \ref{sec:pf_converse}, with the main idea being  that Corollary \ref{cor:batch_lower}  characterizes how the regret must scale in each batch depending on the pre-specified batch lengths.  We show that no matter how the batch lengths are chosen, there will always be at least one ``bad batch'' (i.e., the batch is too long and/or there were not enough prior samples to reduce the uncertainty a sufficient amount) that incurs the specified lower bound on regret.

Theorem \ref{thm:lower}, together with Theorem \ref{thm:upper} (with $\Psi$ treated as a constant), establishes that the overall cumulative regret for the modified BPE algorithm under the constant $B$ setting is $R_T=\Theta^\ast(T^\frac{1-\eta}{1-\eta^B}(\log T)^{\Theta(d)})$ for the SE kernel and $R_T=\Theta^\ast\big(T^\frac{1-\eta}{1-\eta^B}\big)$ for the Mat\'ern kernel.  Thus, we have established matching behavior in the upper and lower bounds.


\section{EXPERIMENTS}

While the goals of this paper are theoretical in nature, here we support our findings with some basic proof-of-concept experiments.  We emphasize that we do not attempt to claim state-of-the-art practical performance against other batch algorithms.

We produce two synthetic 2D functions and execute our algorithm with different values of $B$. We fix $|\mathcal{X}| = 2500$ points by evenly splitting the domain into a $50 \times 50$ grid. We let $T=1000$ and $\sigma=0.02$ and use the SE kernel with $l=0.5$ as prior. For each function, we consider five optimization algorithms: 
\begin{itemize}
    \item[(1)] Orig-BPE: Original BPE with batch sizes as stated in Algorithm \ref{algo:bpe}.
    \item[(2-4)]  $\{3,4,6\}$-BPE: modified BPE with $B=3,4,6$ respectively and batch sizes as stated in Theorem \ref{thm:upper} but with the logarithmic term ignored.  As a slight practical variation, after forming the initial batch lengths $N_1,\dotsc,N_B$, we normalize them by $\frac{T}{\sum_{i=1}^B N_i}$ to maintain an overall length of $T$.
    \item[(5)] GP-UCB, as a representative non-batch algorithm \citep{Sri09}.
\end{itemize}

\begin{figure}[t!]
\centering
\minipage[t]{0.25\textwidth}
	\includegraphics[width=\linewidth]{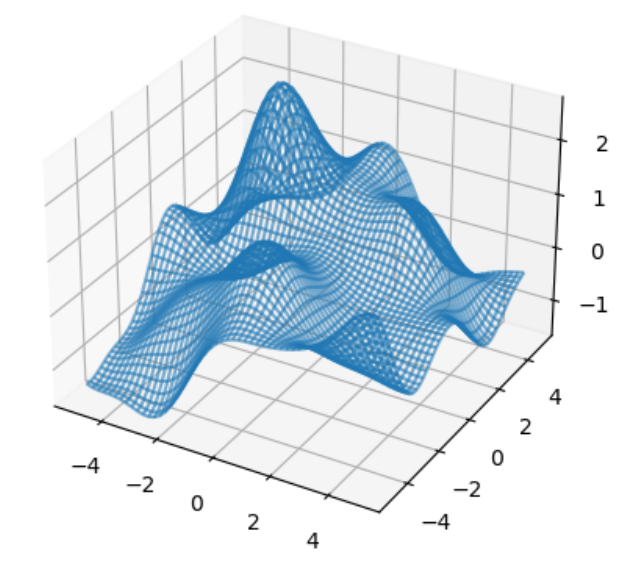}
	\caption{$f_1$}
	\label{fig:f1}
\endminipage
\minipage[t]{0.25\textwidth}
	\includegraphics[width=\linewidth]{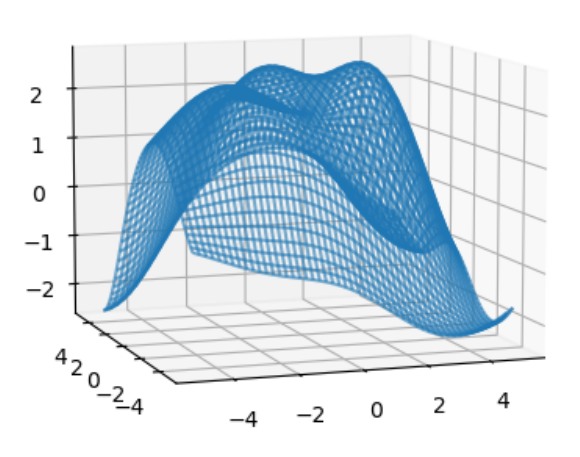}
	\caption{$f_2$}
	\label{fig:f2}
\endminipage
\end{figure}

\begin{figure*}[t!]
\minipage[t]{0.5\textwidth}
	\vspace*{-5mm}
	\includegraphics[width=\linewidth]{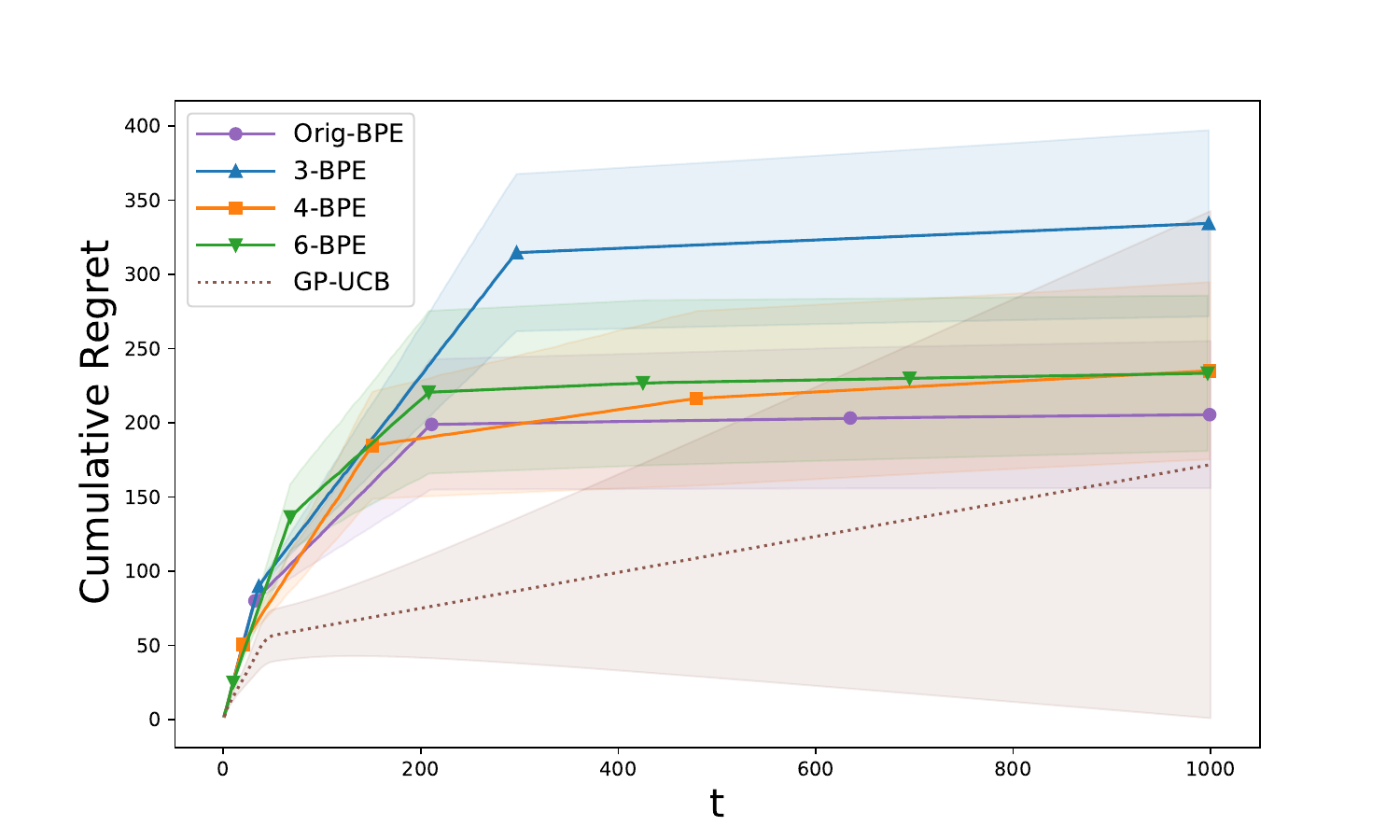}
	\vspace*{-8mm}
  	\caption{$f_1$ with $\beta=2$}
  	\label{fig:f1_2_reset_bpe}
\endminipage
\minipage[t]{0.5\textwidth}
	\vspace*{-5mm}
	\includegraphics[width=\linewidth]{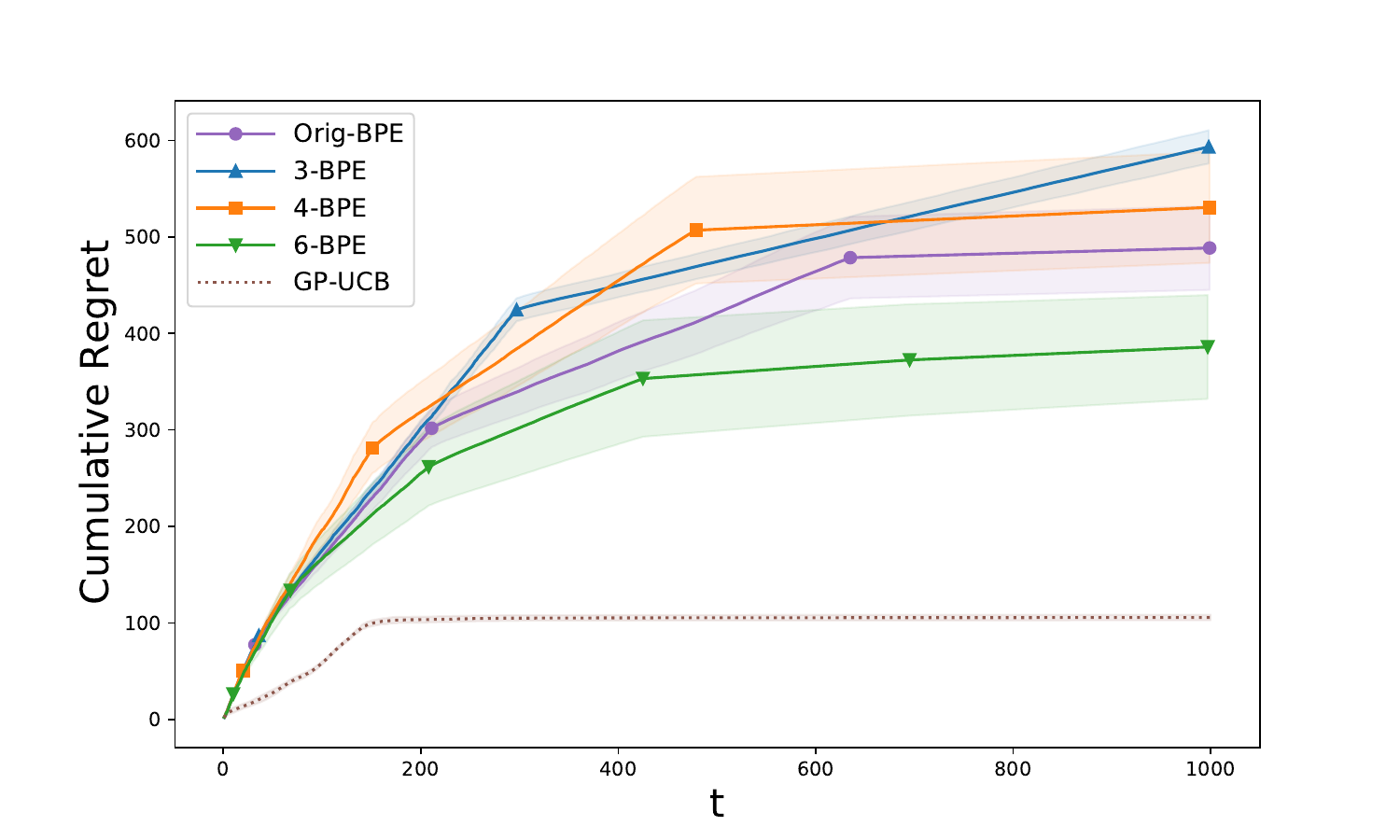}
	\vspace*{-8mm}
	\caption{$f_2$ with $\beta=6$}
	\label{fig:f2_6_reset_bpe}
\endminipage
\end{figure*}

\begin{figure*}[t!]
\minipage[t]{0.5\textwidth}
  \includegraphics[width=\linewidth]{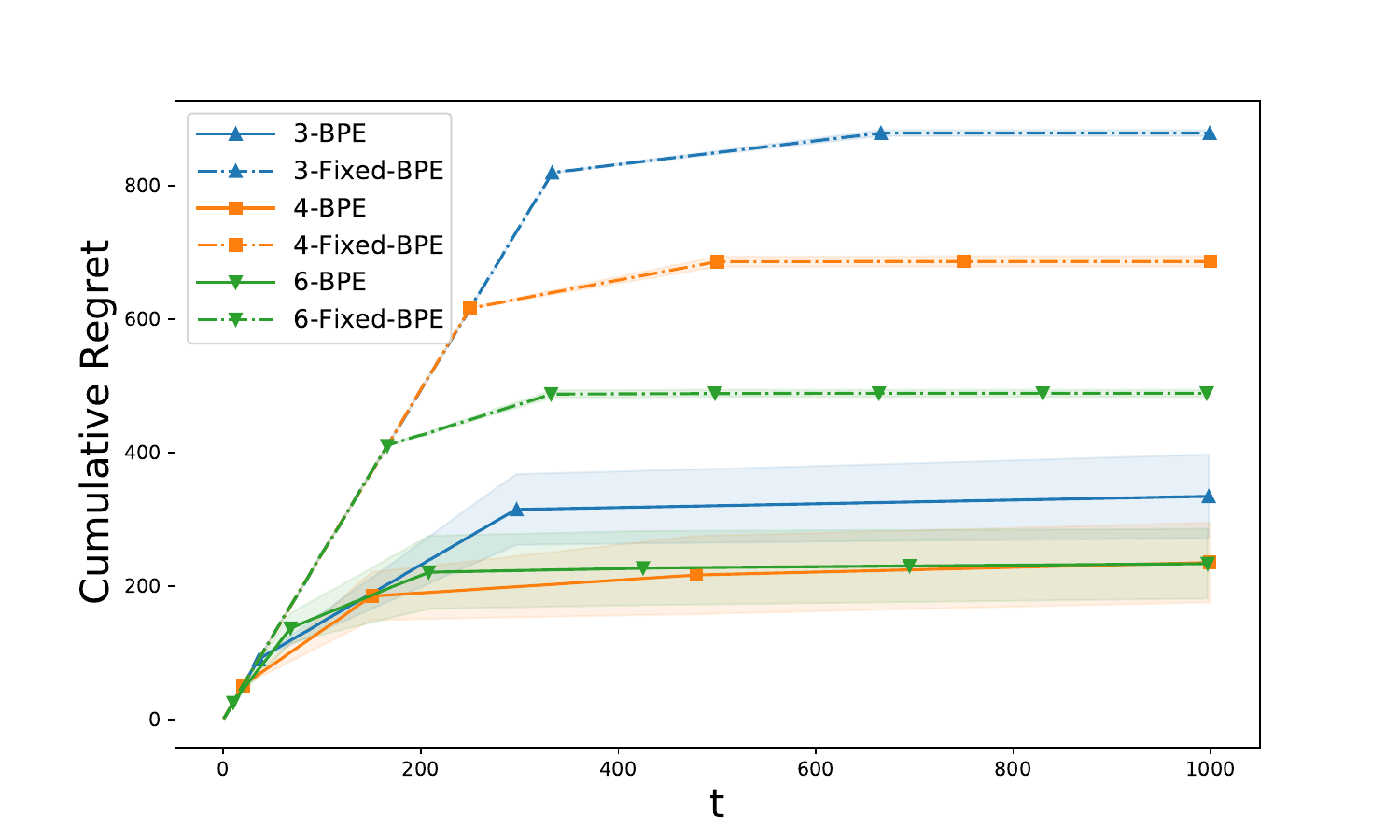}
  \vspace*{-8mm}
  \caption{BPE and Fixed-BPE on $f_1$ with $\beta=2$}
  \label{fig:f1_2_reset_fixed}
\endminipage
\minipage[t]{0.5\textwidth}
	\includegraphics[width=\linewidth]{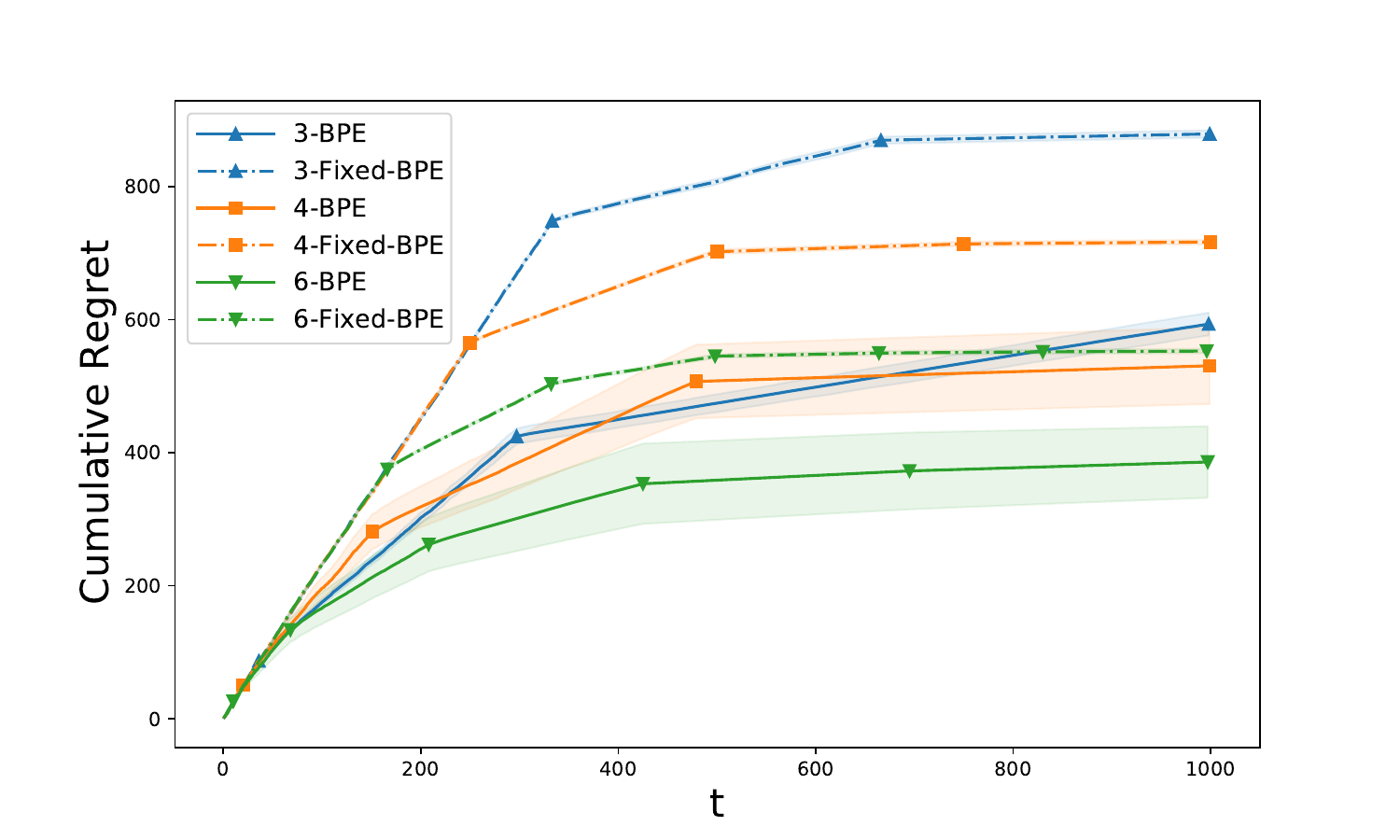}
	\vspace*{-8mm}
	\caption{BPE and Fixed-BPE on $f_2$ with $\beta=6$}
	\label{fig:f2_6_reset_fixed}
\endminipage
\end{figure*}

\begin{figure*}[t!]
\minipage[t]{0.5\textwidth}
 	\includegraphics[width=\linewidth]{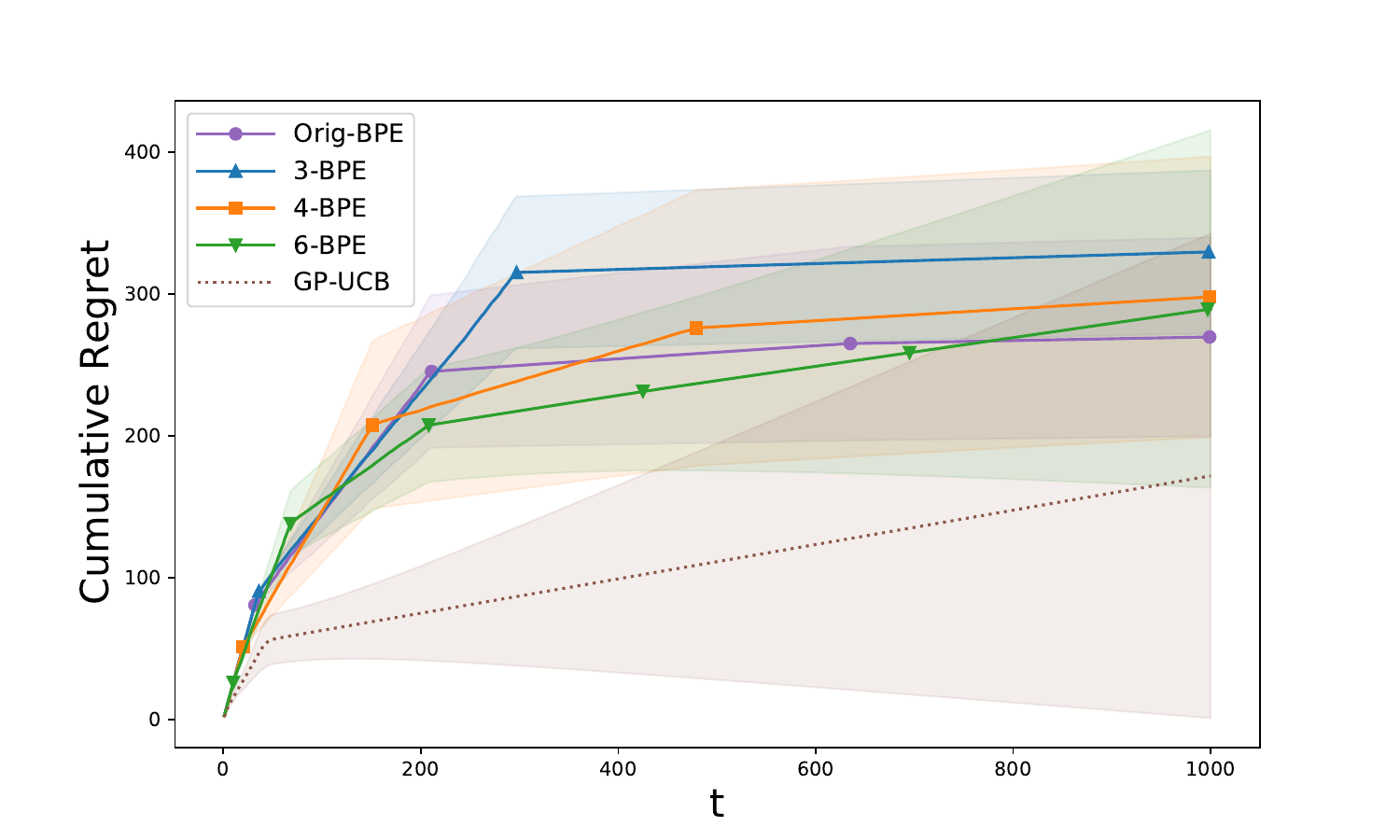}
  	\vspace*{-8mm}
  	\caption{$f_1$ with $\beta=2$ and full posterior}
  	\label{fig:f1_2_full_bpe}
\endminipage
\minipage[t]{0.5\textwidth}
	\includegraphics[width=\linewidth]{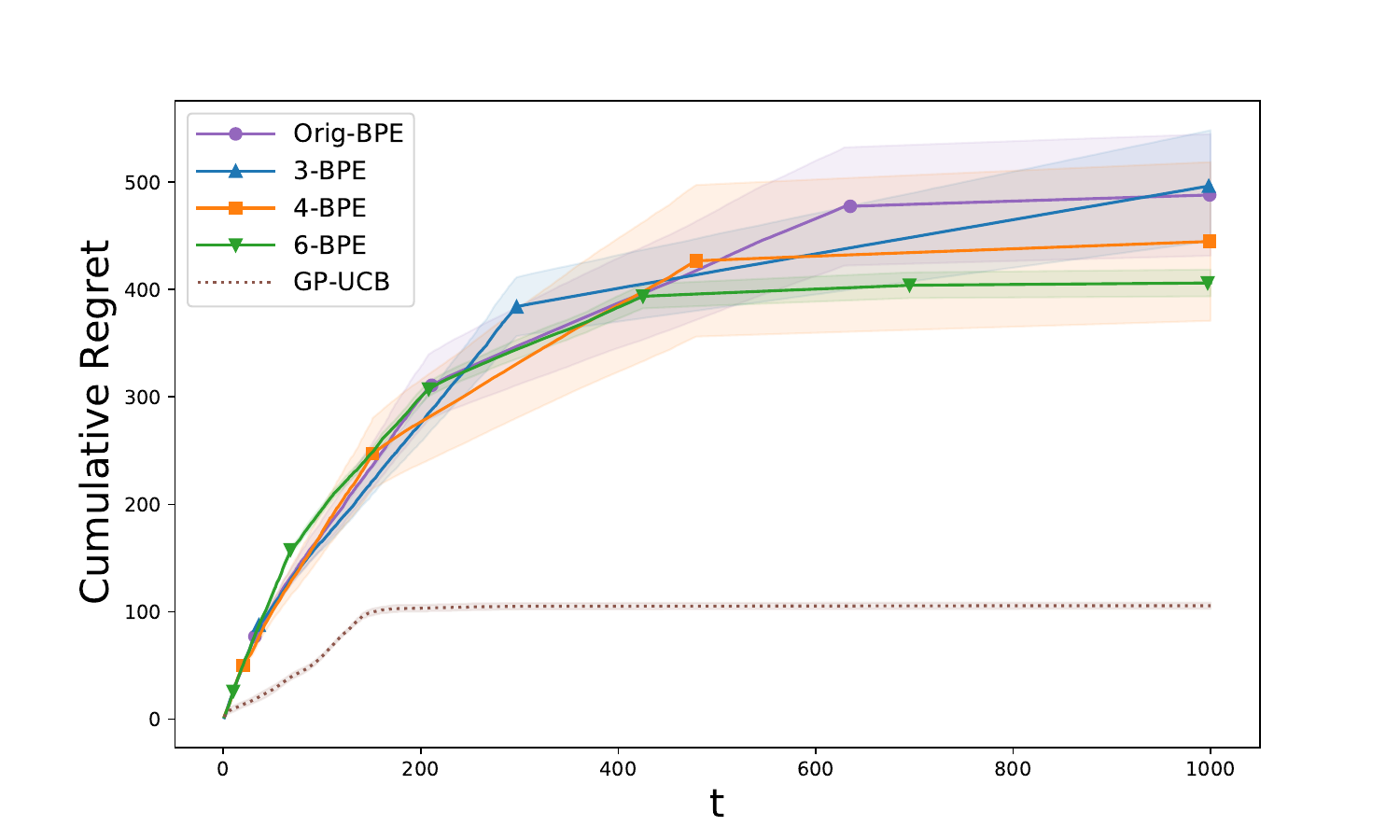}
	\vspace*{-8mm}
	\caption{$f_2$ with $\beta=6$ and full posterior}
	\label{fig:f2_6_full_bpe}
\endminipage
\end{figure*}

\begin{figure*}[t!]
\minipage[t]{0.5\textwidth}
	\includegraphics[width=\linewidth]{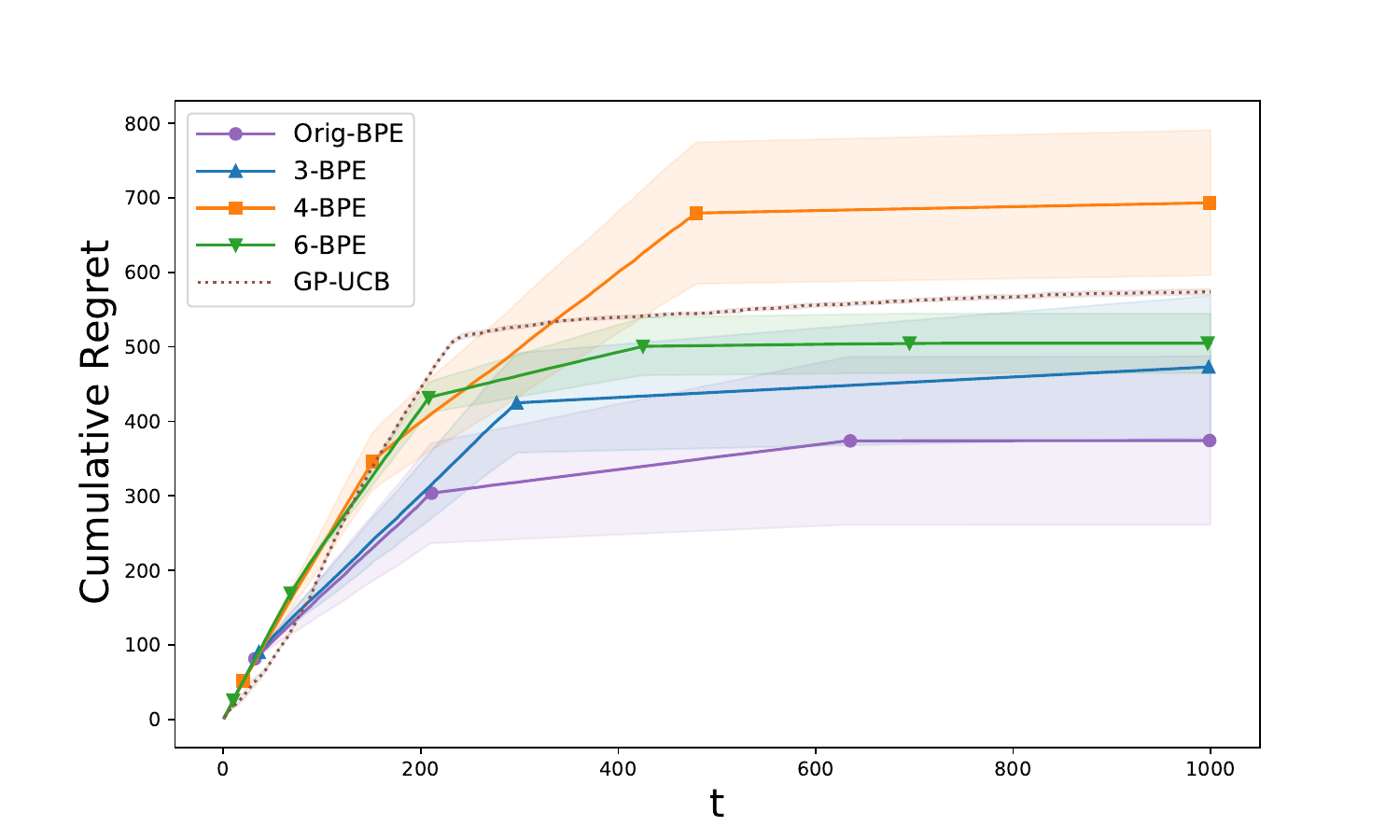}
	\vspace*{-8mm}
	\caption{$f_1$ with $\beta_i=3\log(2i)$ and full posterior}
	\label{fig:f1_log_full_bpe}
\endminipage
\minipage[t]{0.5\textwidth}
	\includegraphics[width=\linewidth]{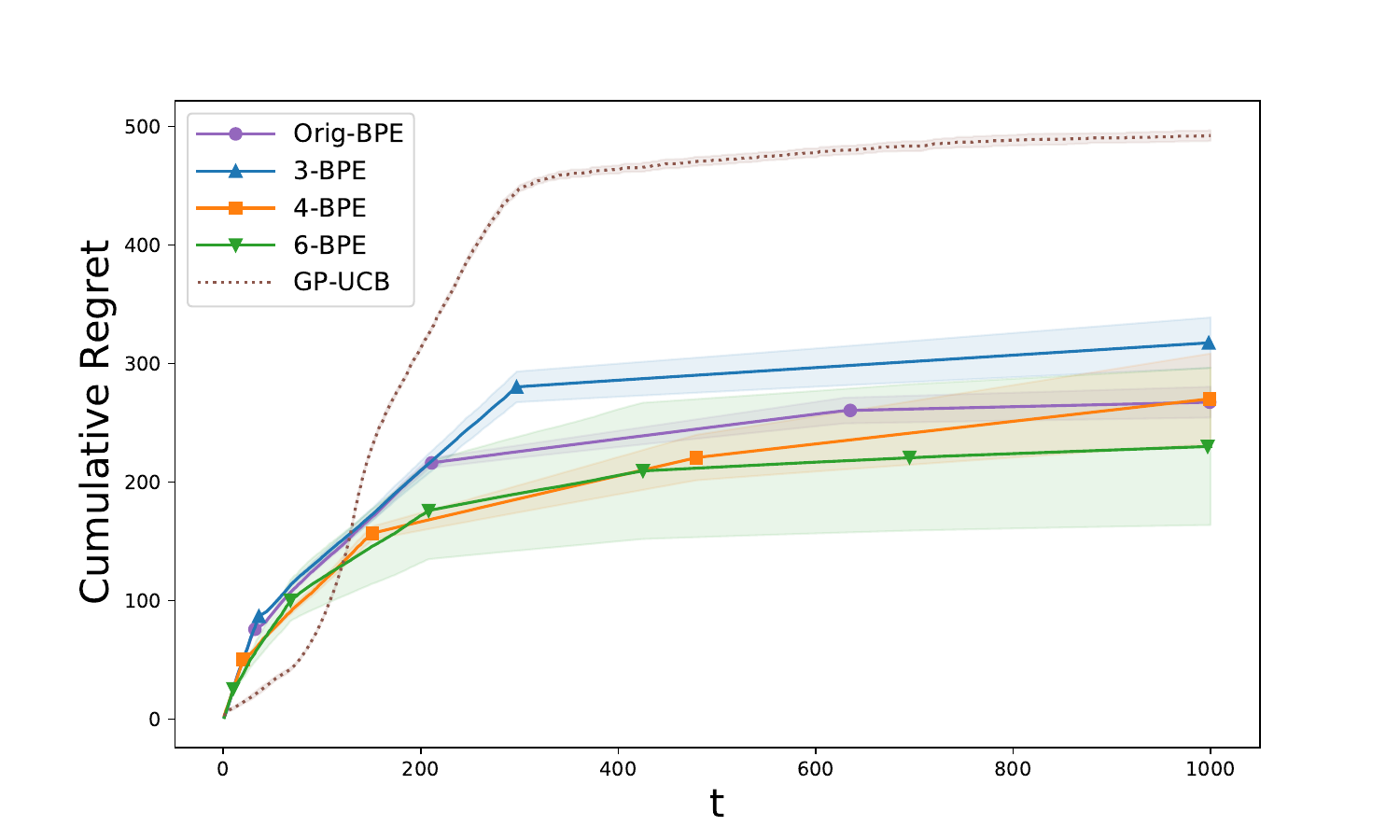}
	\vspace*{-8mm}
	\caption{$f_2$ with $\beta_i=3\log(2i)$ and full posterior}
	\label{fig:f2_log_full_bpe}
\endminipage
\end{figure*}

We use $f_1$ shown in Figure \ref{fig:f1}, a function with one obvious maximum, and $f_2$ shown in Figure \ref{fig:f2}, a function with multiple near-optimal local maxima.  We produce our results by performing $10$ trials and plotting the average cumulative regret, with error bars indicating half a standard deviation. As is common in the literature, instead of relying on the theoretical choice of $\beta$, we set it based on a minimal amount of tuning, choosing $\beta = 2$ for $f_1$ (less exploration needed to find the clear peak) and $\beta = 6$ for $f_2$ (encouraging slightly more exploration).

{\bf Comparison of algorithms.} We compare the algorithms for $f_1$ and $f_2$ in Figures \ref{fig:f1_2_reset_bpe} and \ref{fig:f2_6_reset_bpe}, respectively.  While the curves for $B \in \{3,4,6\}$ are often close, there is a clear trend that smaller $B$ tends to incur higher regret, which is consistent with our theory.  We also observe that $6$-BPE can outperform Orig-BPE, which is natural because the latter only uses 4 batches when $T = 1000$ (the growth of $\log \log T$ is extremely slow).  GP-UCB unsurprisingly has the smallest regret, as it has the benefit of being fully sequential, though the large error bars in Figure \ref{fig:f1_2_reset_bpe} indicate that compared to BPE, it can suffer more from ``failure runs'' when $\beta$ is chosen too aggressively (small).

{\bf Fixed vs.~varying batch lengths.} In Figures \ref{fig:f1_2_reset_fixed} and \ref{fig:f2_6_reset_fixed}, we compare our varying batch lengths to fixed batch lengths.  We observe that fixed batch lengths tend to be highly suboptimal, though the gap shrinks as $B$ increases.  This supports the fact that varying batch lengths are essential for obtaining near-optimal regret guarantees, as discussed in the introduction.

{\bf Partial vs.~full posterior.} Recall that our theoretical analysis crucially relied on ignoring previous batches in order to apply Lemma \ref{lem:conf}.  Next, we investigate whether using all sampled points to form a ``full posterior'' can also be effective in practice.  In this case, prior work suggests that $\beta$ should logarithmically increase with the batch index, so along with the above choices of $\beta=2$ and $\beta =6$, we consider the use of $\beta_i=3\log(2i)$ in the $i$-th batch, which is analogous to the choices made in existing works (e.g., \citep{Sri09,Rol18}).  The results are shown in Figures \ref{fig:f1_2_full_bpe} and \ref{fig:f2_6_full_bpe} (fixed $\beta$) and Figures \ref{fig:f1_log_full_bpe} and \ref{fig:f2_log_full_bpe} (growing $\beta_i$).

The curves in Figures \ref{fig:f1_2_full_bpe} and \ref{fig:f2_6_full_bpe} generally reach similar regret values as Figures \ref{fig:f1_2_reset_bpe} and \ref{fig:f2_6_reset_bpe}, suggesting that using the full posterior can also be effective, though establishing the corresponding theory appears to be difficult.  Another viewpoint on this finding is that although discarding previously-collected data may seem wasteful, it does not appear to harm the performance here.  In these figures, we also observe an effect of the various versions of BPE moving closer together, lower values of $B$ having reduced regret, but the higher values of $B$ having increased regret.

Finally, from Figures \ref{fig:f1_log_full_bpe} and \ref{fig:f2_log_full_bpe}, when we use $\beta_i=3\log(2i)$, the ordering of the algorithms can change significantly.  This is because of two distinct factors impacting the algorithms in different ways:  Increasing $B$ has its obvious benefit of gathering information more frequently, but it also has the effect of making $\beta_i$ grow larger than the ideal value.  An extreme case of this is GP-UCB in Figure \ref{fig:f2_log_full_bpe}, where the regret is significantly higher due to $\beta_i$ eventually growing as large as $3\log(2000) \approx 22.8$.  Of course, further tuning $\beta$ separately for each algorithm would alleviate this issue, but may not always be practical.

\section{CONCLUSION}

We introduced BPE algorithm for GP bandit optimization with few batches. When the number of batch is not specified by the problem, BPE achieves $O^\ast(\sqrt{T\gamma_T})$ cumulative regret with $O(\log\log T)$ batches. When the number of batches is a small constant specified in the problem, we proposed BPE with modified batch size arrangement. We also provided a lower bound for general batch algorithms with a given number of batches.

\subsubsection*{Acknowledgements}
This work was supported by the Singapore National Research Foundation (NRF) under grant number R-252-000-A74-281.

\bibliography{FewBatches_References}


\clearpage
\appendix

\thispagestyle{empty}

\onecolumn \makesupplementtitle

All citations in this appendix are to the reference list in the main body.

\section{PROOFS OF MAIN RESULTS}

\subsection{Proof of Theorem \ref{thm:bpe} (Upper Bound with $O(\log \log T)$ Batches)}\label{sec:proof_bpe}

As discussed in the main text, we can apply Corollary \ref{cor:beta} to any single batch due to the manner in which our algorithm ignores previous batches.  To ensure valid confidence bounds simultaneously for all batches with probability at least $1-\delta$, we further replace $\delta$ by $\frac{\delta}{B}$ therein, and apply the union bound.  This yields $\beta = \big( \Psi + \frac{R}{\sqrt{\lambda}} \sqrt{2 \log \frac{|\Xc| B}{\delta}} \big)^2$, as stated in the description of the algorithm.  We henceforth condition on the resulting confidence bounds remaining valid in all batches.  

With such conditioning, we can deduce the theoretical upper bound on the overall cumulative regret by aggregating the cumulative regret of individual batches. For the first batch with $i=1$ and $t \leq N_1$, the simple regret of any selected point is upper bounded as follows:
\begin{align*}
	r_t &= f(x^\ast)-f(x_t) \stackrel{(a)}{\leq} \mathrm{UCB}_1(x^\ast)-\mathrm{LCB}_1(x_t)\leq  (\mu_0+\sqrt{\beta}\sigma_0)-(\mu_0-\sqrt{\beta}\sigma_0) \stackrel{(b)}{\leq} 2\sqrt{\beta},
\end{align*}
where (a) uses the validity of the confidence bounds, and (b) follows since $\mu_0(\cdot)=0$ and $\sigma_0(\cdot)=k(\cdot,\cdot)\leq 1$. Hence, with $N_1=\big\lceil\sqrt{T}\big\rceil<\sqrt{T}+1$, the cumulative regret of the first batch is
\begin{align}
	R^1=\sum_{t=1}^{N_1}r_t\leq 2N_1\sqrt{\beta}=O(\sqrt{T\beta}). \label{eq:R1}
\end{align}
From the second batch onwards with $i\geq 2$ and $t>N_1$, we make use of the following lemma, which is stated in a generic form (i.e., it should initially be read without regard to using batches).

\begin{lem} \label{lem:elim}
\textup{\citep[Appendix B.5]{Cai21}}\label{lem:r_t}
    Suppose that we run an algorithm that, at each time $t$, samples the point $x_t$ with the highest value of $\sigma_{t-1}(x_t)$.   Moreover, suppose that after $\tau$ rounds of sampling, all points in the domain satisfy a confidence bound of the form $|f(x)-\mu_\tau(x)|\leq\sqrt{\beta}\sigma_\tau(x)$ for some $\beta > 0$.  Then, for any point $x$ whose UCB score is higher than the maximum LCB score, it must hold that
    \begin{equation}
        \big| f(x^*)-f(x)\big| \leq2\sqrt{\frac{C_1\gamma_{\tau} \beta}{\tau}},
    \end{equation}
    where $C_1=8/\log(1+\lambda^{-1})$, and $x^*$ is any maximizer of $f$.
\end{lem}

Since Lemma \ref{lem:elim} is a known result, we do not provide a proof, but we briefly note that it follows easily from the following standard facts:
\begin{itemize}
    \item[(i)] As long as the confidence bounds are valid, $x^*$ is not eliminated (in the sense of the UCB vs.~LCB comparison mentioned above).
    \item[(ii)] Under the maximum-variance selection rule, upper bounding the sampled point's variance amounts to upper bounding the variance of all points.
    \item[(iii)] The sum of variances across all sampled points can be bounded in terms of $\gamma_T$ using techniques dating back to \citep{Sri09} that are now very standard.
    \item[(iv)] If the confidence width uniformly below $\eta$, then all non-eliminated points must have a function value within $O(\eta)$ of the maximum.
\end{itemize}

Using Lemma \ref{lem:elim}, we deduce that after the $N_{i-1}$ points in the $(i-1)$-th batch are sampled, every point $x_t$ in the $i$-th batch has regret $r_t$ upper bounded as follows:
\begin{align}
	r_t \leq 2\sqrt{\frac{C_1\gamma_{N_{i-1}}\beta}{N_{i-1}}}
    \leq 2\sqrt{\frac{C_1\gamma_T\beta}{N_{i-1}}}, \label{eq:r_t}
\end{align}
where we used the fact that $N_{i-1}$ points are sampled in the previous batch, and applied $\gamma_{N_{i-1}}\leq\gamma_T$ by the monotonicity of $\gamma_t$.

The choice $N_i=\big\lceil\sqrt{TN_{i-1}}\big\rceil$ implies $N_i<\sqrt{TN_{i-1}}+1$, and therefore
\begin{align*}
    \frac{N_i}{\sqrt{N_{i-1}}}&<\sqrt{T}+\frac{1}{\sqrt{N_{i-1}}}\leq \sqrt{T}+\frac{1}{\sqrt{N_{1}}}\leq\sqrt{T}+T^{-\frac{1}{4}},
\end{align*}
which implies that the cumulative regret of batch $i$ is
\begin{align*}
	R^i&=\sum_{t=t_{i-1}+1}^{t_i}r_t\leq2 N_i\sqrt{\frac{C_1\gamma_T\beta}{N_{i-1}}}\leq2\Big(\sqrt{T}+T^{-\frac{1}{4}}\Big)\sqrt{C_1\gamma_T\beta}=O(\sqrt{T\gamma_T\beta}).
\end{align*}
Combining the regret bounds for the batches, it holds with probability at least $1-\delta$ that the cumulative regret incurred up to time $T$ satisfies
\begin{align*}
	R_T&=R^1+\sum_{i=2}^B R^i= O(B\sqrt{T\gamma_T\beta}) =O\big((\log\log T)\sqrt{T\gamma_T\beta}\big)=O^\ast(\Lambda\sqrt{T\gamma_T}),
\end{align*}
where $\Lambda=\Psi+\sqrt{\log\frac{|\Xc|}{\delta}}$ (the $B = O(\log\log T)$ inside the logarithm can be absorbed into the $O^*(\cdot)$ notation, which hides dimension-independent log factors). The kernel-specific bounds then follow from \eqref{eq:se} and \eqref{eq:mat}.

\subsection{Proof of Theorem \ref{thm:upper} (Upper Bound with a Constant Number of Batches)} \label{sec:pf_fixed_B}

We again condition on the confidence bounds holding, and compute the regret upper bound for each individual batch. For the first batch $i=1$, same as \eqref{eq:R1}, we have
\begin{align*}
    R^1 \leq 2\sqrt{\beta}N_1 = O(\sqrt{\beta}N_1).
\end{align*}
For batch $i=2,\dots,B-1$, we have from Lemma \ref{lem:elim} (similarly to \eqref{eq:r_t}) that
\begin{align}
R^i &\leq 2N_i\sqrt{\frac{C_1\beta\gamma_{N_{i-1}}}{N_{i-1}}} \label{eq:Ri_1}\\
&=O\Bigg(N_i\sqrt{\frac{\beta\gamma_{N_{i-1}}}{N_{i-1}}}\Bigg). \label{eq:Ri_2}
\end{align}

For the SE kernel, we proceed as follows:
\begin{align}
    R^i&=O\Bigg(N_i\sqrt{\frac{\beta\gamma_T}{N_{i-1}}}\Bigg)\\
    &\leq O^\ast\Bigg(N_i\sqrt{\frac{\beta (\log T)^d}{N_{i-1}}}\Bigg) \label{eq:Ri_se_1}\\
    &\leq O^\ast(\sqrt{\beta}N_1), \label{eq:Ri_se_2}
\end{align}
where:
\allowdisplaybreaks
\begin{itemize}
    \item \eqref{eq:Ri_se_1} follows from \eqref{eq:se};
    \item \eqref{eq:Ri_se_2} follows since
    \begin{align}
    &N_i(\log T)^\frac{d}{2}(N_{i-1})^{-\frac{1}{2}} \nonumber \\
    &= \Bigg\lceil \Bigg(\frac{T}{(\log T)^d}\Bigg)^\frac{1-\eta^i}{1-\eta^B}(\log T)^d\Bigg\rceil(\log T)^{d\eta}(N_{i-1})^{-\eta} \label{eq:Ri_se_3} \\
    &=\Bigg\lceil \Bigg(\frac{T}{(\log T)^d}\Bigg)^\frac{1-\eta}{1-\eta^B}(\log T)^d \cdot \Bigg(\frac{T}{(\log T)^d}\Bigg)
    ^{\frac{1-\eta^{i-1}}{1-\eta^B}\eta} \Bigg\rceil(\log T)^{d\eta}(N_{i-1})^{-\eta} \label{eq:Ri_se_4} \\
    &\leq\Bigg\lceil \Bigg(\frac{T}{(\log T)^d}\Bigg)^\frac{1-\eta}{1-\eta^B}(\log T)^d \Bigg\rceil\cdot \Bigg\lceil\Bigg(\frac{T}{(\log T)^d}\Bigg)^{\frac{1-\eta^{i-1}}{1-\eta^B}\eta} \Bigg\rceil(\log T)^{d\eta}(N_{i-1})^{-\eta}  \label{eq:Ri_se_5}\\
    &\leq\Bigg\lceil \Bigg(\frac{T}{(\log T)^d}\Bigg)^\frac{1-\eta}{1-\eta^B}(\log T)^d \Bigg\rceil\cdot \Bigg(\Bigg\lceil\Bigg(\frac{T}{(\log T)^d}\Bigg)^{\frac{1-\eta^{i-1}}{1-\eta^B}\eta}(\log T)^{d\eta}  \Bigg\rceil + (\log T)^{d\eta}\Bigg) (N_{i-1})^{-\eta} \label{eq:Ri_se_6}\\
    &\leq\Bigg\lceil\Bigg(\frac{T}{(\log T)^d}\Bigg)^\frac{1-\eta}{1-\eta^B}(\log T)^d \Bigg\rceil\cdot \Bigg(\Bigg\lceil\Bigg(\frac{T}{(\log T)^d}\Bigg)^{\frac{1-\eta^{i-1}}{1-\eta^B}}(\log T)^d  \Bigg\rceil^\eta+1+(\log T)^{d\eta}\Bigg)(N_{i-1})^{-\eta}\label{eq:Ri_se_7} \\
    & = N_1\big((N_{i-1})^\eta+1+(\log T)^{d\eta}\big)(N_{i-1})^{-\eta} \label{eq:Ri_se_8}\\
    &= O(N_1),\label{eq:Ri_se_9}
\end{align}
and where:
\begin{itemize}
    \item \eqref{eq:Ri_se_3} follows from the definitions of $N_i$ and $\eta=\frac{1}{2}$;
    \item \eqref{eq:Ri_se_4} follows since $1-\eta^i=(1-\eta)+(1-\eta^{i-1})\eta$;
    \item \eqref{eq:Ri_se_5} follows since $\lceil ab \rceil \leq \lceil a \rceil\cdot\lceil b \rceil$ for $a,b>0$;
    \item \eqref{eq:Ri_se_6} follows since $\lceil a \rceil\cdot b \leq \lceil ab \rceil + b$ for $a,b>0$;
    \item \eqref{eq:Ri_se_7} follows since $\lceil a^\eta \rceil \leq \lceil a \rceil ^\eta + 1$ for $a>0,0<\eta<1$;
    \item \eqref{eq:Ri_se_8} follows from the definitions of $N_1$ and $N_{i-1}$;
    \item \eqref{eq:Ri_se_9} follows since $(\log T)^{d\eta} = {\rm poly}(\log T)$ is asymptotically negligible compared to $(N_{i-1})^{\eta} = T^{\Theta(1)}$ (recall that $d$ is considered constant with respect to $T$).
\end{itemize}
\end{itemize}

For the Mat\'ern kernel, we proceed similarly:
\begin{align}
    R^i&=O\Bigg(N_i\sqrt{\frac{\beta\gamma_{N_{i-1}}}{N_{i-1}}}\Bigg)\\
    &\leq O^\ast\Bigg(N_i\sqrt{\frac{\beta(N_{i-1})^{1-2\eta}}{N_{i-1}}}\Bigg) \label{eq:Ri_mat_1}\\
    &=O^\ast\Big(\sqrt{\beta}N_i (N_{i-1})^{-\eta} \Big) \label{eq:Ri_mat_2}\\
    &\leq O^\ast(\sqrt{\beta}N_1),\label{eq:Ri_mat_3}
\end{align}
where:
\begin{itemize}
    \item \eqref{eq:Ri_mat_1} follows from $O(\gamma_{N_{i-1}})=O^\ast((N_{i-1})^{1-2\eta})$ by applying $\frac{d}{2\nu+d}=1-2\eta$ with $\eta=\frac{\nu}{2\nu+d}$ in \eqref{eq:mat};
    \item \eqref{eq:Ri_mat_3} follows via similar steps to \eqref{eq:Ri_se_3}--\eqref{eq:Ri_se_9}:
    \begin{align}
        N_i(N_{i-1})^{-\eta}&=\Big\lceil T^\frac{1-\eta^i}{1-\eta^B}\Big\rceil(N_{i-1})^{-\eta} \label{eq:mat_start}\\
        &=\Big\lceil T^\frac{1-\eta}{1-\eta^B}\cdot T^{\frac{1-\eta^{i-1}}{1-\eta^B}\eta}\Big\rceil(N_{i-1})^{-\eta}\\
        &\leq\Big\lceil T^\frac{1-\eta}{1-\eta^B}\Big\rceil\cdot\Big\lceil T^{\frac{1-\eta^{i-1}}{1-\eta^B}\eta}\Big\rceil(N_{i-1})^{-\eta}\\
        &\leq \Big\lceil T^\frac{1-\eta}{1-\eta^B}\Big\rceil \Big(\Big\lceil T^{\frac{1-\eta^{i-1}}{1-\eta^B}}\Big\rceil^\eta+1\Big)(N_{i-1})^{-\eta}\\
        &=N_1((N_{i-1})^\eta+1)(N_{i-1})^{-\eta}\\
        &\leq N_1+(N_1)^{1-\eta}\\
        &=O(N_1).\label{eq:mat_end}
    \end{align}
\end{itemize}
For the last batch $i=B$ with the SE kernel, we have from \eqref{eq:Ri_se_1} that
\begin{align*}
    R^B&=O^\ast\Bigg(N_B\sqrt{\frac{\beta(\log T)^d}{N_{B-1}}}\Bigg)\leq O^\ast\Bigg(T\sqrt{\frac{\beta(\log T)^d}{N_{B-1}}}\Bigg)=O^\ast(\sqrt{\beta}N_1),
\end{align*}
where the last step follows by substituting $i=B$ into the right hand side of \eqref{eq:Ri_se_3} and then substituting \eqref{eq:Ri_se_9}.

Similarly, for the last batch $i=B$ with the Mat\'ern kernel, we have from \eqref{eq:Ri_mat_3} that
\begin{align*}
    R^B&=O^\ast\Big(\sqrt{\beta}N_B(N_{B-1})^{-\eta}\Big)\leq O^\ast\Big(\sqrt{\beta}T(N_{B-1})^{-\eta}\Big) =O^\ast(\sqrt{\beta}N_1),
\end{align*}
where the last step follows by substituting $i=B$ into the right hand side of \eqref{eq:mat_start} and then applying \eqref{eq:mat_end}.

We have thus shown that the assignment of batch sizes in Theorem \ref{thm:upper} makes each batch have the same regret upper bound $O^\ast(\sqrt{\beta}N_1)$ for both kernels, and hence, we have the overall cumulative regret is
\begin{align*}
    R_T=\sum_{i=1}^B R^i \leq O^\ast(B\sqrt{\beta}N_1)=O^\ast(\sqrt{\beta}N_1).
\end{align*}
The proof is completed by substituting the choice of $N_1$.

\subsection{Proof of Theorem \ref{thm:lower} (Lower Bound with a Constant Number of Batches)} \label{sec:pf_converse}

For the SE kernel with $\eta=\frac{1}{2}$, it is convenient to rewrite the lower bound in Corollary \ref{cor:batch_lower} as $R^m=\Omega\big(\big(\frac{T}{(\log T)^{d/2}}\big)^{b-\eta a}(\log T)^{d/2}\big)$. For any batch algorithm with time horizon $T$, number of batches $B$, and any pre-specified batch size arrangement, we will always fall into one of the $B$ cases stated as follows:
\begin{enumerate}
    \item With $t_B=T$ given, if $t_{B-1}< \big(\frac{T}{(\log T)^{d/2}}\big)^\frac{1-\eta^{B-1}}{1-\eta^B}(\log T)^{d/2}$, then by applying Corollary \ref{cor:batch_lower} with $a=\frac{1-\eta^{B-1}}{1-\eta^B}$ and $b=1=\frac{1-\eta^B}{1-\eta^B}$, there exists $f\in\mathcal{F}_k(\Psi)$ such that $R_T \geq R^B = \Omega\big(\big(\frac{T}{(\log T)^{d/2}}\big)^\frac{1-\eta}{1-\eta^B}(\log T)^{d/2}\big)$ with probability at least $\frac{3}{4}$.
    \item Otherwise, if $t_{B-1}\geq \big(\frac{T}{(\log T)^{d/2}}\big)^\frac{1-\eta^{B-1}}{1-\eta^B}(\log T)^{d/2}$ and $t_{B-2}<\big(\frac{T}{(\log T)^{d/2}}\big)^\frac{1-\eta^{B-2}}{1-\eta^B}(\log T)^{d/2}$, then by applying Corollary \ref{cor:batch_lower} with $a=\frac{1-\eta^{B-2}}{1-\eta^B}$ and $b=\frac{1-\eta^{B-1}}{1-\eta^B}$, there exists $f\in\mathcal{F}_k(\Psi)$ such that $R_T \geq R^{B-1} = \Omega\big(\big(\frac{T}{(\log T)^{d/2}}\big)^\frac{1-\eta}{1-\eta^B}(\log T)^{d/2}\big)$ with probability at least $\frac{3}{4}$, since $b-\frac{a}{2}=b-\eta a=\frac{1-\eta}{1-\eta^B}$.
    \item Similar to steps 1 and 2, for $m=B-2,B-3,\dots,2$, we iteratively argue as follows: If $t_{m}\geq \big(\frac{T}{(\log T)^{d/2}}\big)^\frac{1-\eta^{m}}{1-\eta^B}(\log T)^{d/2}$ and $t_{m-1}< \big(\frac{T}{(\log T)^{d/2}}\big)^\frac{1-\eta^{m-1}}{1-\eta^B}(\log T)^{d/2}$, then by applying Corollary \ref{cor:batch_lower} with $a=\frac{1-\eta^{m-1}}{1-\eta^B}$ and $b=\frac{1-\eta^m}{1-\eta^B}$, there exists $f\in\mathcal{F}_k(\Psi)$ such that $R_T \geq R^m = \Omega\big(\big(\frac{T}{(\log T)^{d/2}}\big)^\frac{1-\eta}{1-\eta^B}(\log T)^{d/2}\big)$ with probability at least $\frac{3}{4}$.
    \item Now the only case remaining is when $t_1=N_1\geq \big(\frac{T}{(\log T)^{d/2}}\big)^\frac{1-\eta}{1-\eta^B}(\log T)^{d/2}$.  The idea here is that the regret of points in the first batch is typically $\Omega(1)$ due to the fact that no observations have been made.  In more detail, we can consider $f$ being any function with a clear peak of height $\Theta(1)$, such that the majority of the domain's points incur regret $\Omega(1)$.  By considering a class of functions of this type\footnote{For brevity, we have not described a specific ``hard'' function class here, but formally, the class of functions used in previous lower bounds \citep{Sca17a,Cai20} would suffice.  The idea is that in the same way that Corollary \ref{cor:batch_lower} derives a cumulative regret bound from Lemma \ref{lem:lower}'s simple regret bound, here we get an $\Omega(N_1)$ cumulative regret bound directly from the trivial $\Omega(1)$ simple regret bound when {\em no points are observed}.} with the peak location shifted throughout the domain, we readily obtain that there exists $f\in\mathcal{F}_k(\Psi)$ such that $R^1 = \Omega(N_1)$, and hence $R_T \geq R^1=\Omega\big(\big(\frac{T}{(\log T)^{d/2}}\big)^\frac{1-\eta}{1-\eta^B}(\log T)^{d/2}\big)$, with probability at least $\frac{3}{4}$.
\end{enumerate}
Since $\frac{d}{2}(1-\frac{1-\eta}{1-\eta^B})=\frac{d}{2}(\frac{\eta-\eta^B}{1-\eta^B})$, we have $R_T=\Omega\big(\big(\frac{T}{(\log T)^{d/2}}\big)^\frac{1-\eta}{1-\eta^B}(\log T)^{d/2}\big)=\Omega\big(T^\frac{1-\eta}{1-\eta^B}(\log T)^{\frac{d(\eta-\eta^B)}{2(1-\eta^B)}}\big)$ for some $f\in\mathcal{F}_k(\Psi)$ with probability at least $\frac{3}{4}$ in all cases, and the claimed lower bound follows.

Similarly, for the Mat\'ern kernel with $\eta=\frac{\nu}{2\nu+d}$, the lower bound in Corollary \ref{cor:batch_lower} can be rewritten as $R^m=\Omega\big(T^{b-\eta a}\big)$. For any batch algorithm with time horizon $T$, number of batches $B$, and any pre-specified batch size arrangement, we will always fall into one of the $B$ cases stated as follows:
\begin{enumerate}
    \item With $t_B=T$ given, if $t_{B-1}< T^\frac{1-\eta^{B-1}}{1-\eta^B}$, then by applying Corollary \ref{cor:batch_lower} with $a=\frac{1-\eta^{B-1}}{1-\eta^B}$ and $b=1=\frac{1-\eta^B}{1-\eta^B}$, there exists $f\in\mathcal{F}_k(\Psi)$ such that $R_T \geq R^B = \Omega\big(T^\frac{1-\eta}{1-\eta^B}\big)$ with probability at least $\frac{3}{4}$.
    \item Otherwise, if $t_{B-1}\geq T^{\frac{1-\eta^{B-1}}{1-\eta^B}}$ and $t_{B-2}<T^\frac{1-\eta^{B-2}}{1-\eta^B}$, then by applying Corollary \ref{cor:batch_lower} with $a=\frac{1-\eta^{B-2}}{1-\eta^B}$ and $b=\frac{1-\eta^{B-1}}{1-\eta^B}$, there exists $f\in\mathcal{F}_k(\Psi)$ such that $R_T \geq R^{B-1} = \Omega\big(T^\frac{1-\eta}{1-\eta^B}\big)$ with probability at least $\frac{3}{4}$.
    \item Similar to steps 1 and 2, for $m=B-2,B-3,\dots,2$, we iteratively argue as follows: If $t_{m}\geq T^{\frac{1-\eta^m}{1-\eta^B}}$ and $t_{m-1}< T^\frac{1-\eta^{m-1}}{1-\eta^B}$, then by applying Corollary \ref{cor:batch_lower} with $a=\frac{1-\eta^{m-1}}{1-\eta^B}$ and $b=\frac{1-\eta^m}{1-\eta^B}$, there exists $f\in\mathcal{F}_k(\Psi)$ such that $R_T \geq R^m = \Omega\big(T^\frac{1-\eta}{1-\eta^B}\big)$ with probability at least $\frac{3}{4}$.
    \item Now the only case remaining is when $t_1=N_1\geq T^\frac{1-\eta}{1-\eta^B}$.  In this case, by the same argument as the SE kernel above, there exists $f\in\mathcal{F}_k(\Psi)$ such that $R_T \geq R^1=\Omega\big(T^\frac{1-\eta}{1-\eta^B}\big)$ for some $f\in\mathcal{F}_k(\Psi)$ with probability at least $\frac{3}{4}$.
\end{enumerate}
It follows that $R_T=\Omega\big(T^\frac{1-\eta}{1-\eta^B}\big)$ for some $f\in\mathcal{F}_k(\Psi)$ with probability at least $\frac{3}{4}$ in all cases, and the claimed lower bound follows.

\section{DISCUSSION ON LOWER BOUNDS WITH ADAPTIVELY-CHOSEN BATCH LENGTHS} \label{sec:variable}

To discuss the difficulties in establishing tight lower bounds when the batches sizes can be chosen adaptively (as opposed to being pre-specified), it is useful to recap how the algorithm-independent lower bounds are proved in \citep{Sca17a,Cai20}.

\subsection{Background}

The high-level idea in the above-mentioned works is to form a lower bound based on the difficulty of locating a small narrow bump in $[0,1]^d$, where the peak of the bump is $2\epsilon$ for some parameter $\epsilon > 0$, but the vast majority of the space has $f(x) \approx 0$ (or even $f(x) = 0$).  The RKHS norm constraint limits how narrow the bump can be, and accordingly how many of these functions can be ``packed'' into the space while keeping the bump parts non-overlapping.

For the SE and Mat\'ern kernels, the quantities that arise are as follows:
\begin{itemize}
    \item For the SE kernel, the bump width is on the order of $w \sim \frac{1}{ \sqrt{\log\frac{1}{\epsilon}} }$, and so the total number of bumps is on the order of $M \sim \big(\log\frac{1}{\epsilon}\big)^{d/2}$ (constants are omitted in this informal discussion).
    \item For the Mat\'ern kernel, the bump width is on the order of $w \sim \epsilon^{1/\nu}$, yielding $M \sim \big(\frac{1}{\epsilon}\big)^{d/\nu}$.
\end{itemize}
The idea is then that since there are $M$ possible bump locations but the function only peaks at $2\epsilon$ and is nearly zero throughout the domain, around $T \sim \frac{M}{\epsilon^2}$ samples are needed to locate the bump (and thus attain simple regret at most $\epsilon$).  Solving this formula for $\epsilon$ gives the lower bound on the simple regret as a function of $T$.


\subsection{Difficulty in the Batch Setting}

The proof of Theorem \ref{thm:lower} is based on a recursive argument based on how long the batch lengths are:
\begin{itemize}
    \item If the final batch starts too early, then the information gathered in the first $B-1$ batches is insufficient to locate a height-$(\epsilon_{B-1})$ bump (for some suitably-chosen $\epsilon_{B-1}$), so every action played in the final batch is likely to incur $\Omega(\epsilon_{B-1})$ regret.
    \item If the final batch does start sufficiently late to avoid the previous case, but the second-last batch starts too early, then the information gathered in the first $B-2$ batches is insufficient to locate a height-$(\epsilon_{B-2})$ bump, amounting to $\Omega(\epsilon_{B-2})$ regret for actions in the second-last batch.
    \item The previous dot point continues recursively, and the final case is that if the first batch starts too late, then due to having no prior information for the first batch, each selected action will incur $\epsilon_0 = \Omega(1)$ regret.
\end{itemize}
The difficulty is that {\em the lower bound on regret in each of these cases corresponds to a different ``hard subset'' of functions}.  Specifically, the difficult case for the later batches is a  function class with a larger number of functions consisting of shorter and narrower bumps, whereas the difficult case for the earlier batches is a function class with a smaller number of functions consisting of relatively higher and wider bumps.

If the batch lengths are pre-specified, then we can readily choose the hard subset for the appropriate case depending on those lengths.  However, if the batch sizes are chosen adaptively, then the argument appears to be more difficult, because in principle the algorithm could suitably adapt the batch lengths to overcome the difficulty of optimizing a single {\em a priori} specified hard class.

\subsection{Plausibility Argument}

Here we outline a plausibility argument for Theorem \ref{thm:lower} remaining true even in the case of adaptively-chosen batch sizes.  To do so, we let $t_1^*,\dotsc,t_{B}^*$ be the ``critical'' ending times of the batches (with $t_B^* = T$) dictated by the proof of Theorem \ref{thm:lower}, e.g., $t_1^* \sim T^\frac{1-\eta}{1-\eta^B}$ for the Mat\'ern kernel according to the item 4 at the end of Appendix \ref{sec:pf_converse}.  In addition, we let $\Fc_1,\dotsc,\Fc_B$ be the hard function classes associated with batches $1,\dotsc,B$, as discussed above.

The plausibility argument is as follows.  Suppose that the underlying function comes from $\Fc_1 \cup \dotsc \Fc_B$, but is otherwise unknown.  Then:
\begin{itemize}
    \item To handle the possibility that the function is from $\Fc_1$, the first batch length must be no higher than $t_1^*$.
    \item After observing the first batch of size at most $t_1^*$, if the function was not from $\Fc_1$, then no matter what class among $\Fc_2,\dotsc,\Fc_B$ the function was from, not enough samples have been taken to know where the function's bump lies.  This suggests that the algorithm still has no way of reliably distinguishing between $\Fc_2,\dotsc,\Fc_B$, and it must prepare for the possibility that the function lies in $\Fc_2$. Thus, the second batch should finish no later than time $t_2^*$.
    \item This argument continues recursively until the $(B-1)$-th batch finishes no later than time $t_{B-1}^*$, but even in this case, the final batch incurs enough regret to produce the desired lower bound.
\end{itemize}
Formalizing this argument is left for future work.

\end{document}